\documentclass[letterpaper]{article} 
\usepackage{aaai24}  
\usepackage{times}  
\usepackage{helvet}  
\usepackage{courier}  
\usepackage[hyphens]{url}  
\usepackage{graphicx} 
\usepackage{natbib}  
\usepackage{caption} 
\usepackage{algorithm}
\usepackage{algorithmic}

\usepackage{newfloat}
\usepackage{listings}
\DeclareCaptionStyle{ruled}{labelfont=normalfont,labelsep=colon,strut=off} 
\floatstyle{ruled}
\newfloat{listing}{tb}{lst}{}
\floatname{listing}{Listing}
\title{Safely Exploring Novel Actions in Recommender Systems \\via Deployment-Efficient Policy Learning}
\author{
    Haruka Kiyohara\textsuperscript{\rm 1} 
    Yusuke Narita\textsuperscript{\rm 2},
    Yuta Saito\textsuperscript{\rm 3},
    Kei Tateno\textsuperscript{\rm 4},
    Takuma Udagawa\textsuperscript{\rm 4}
}
\affiliations{
    \textsuperscript{\rm 1}Cornell University \\
    \textsuperscript{\rm 2}Yale University \\
    \textsuperscript{\rm 3}Hanjuku-kaso Co., Ltd. \\
    \textsuperscript{\rm 4}Sony Group Corporation \\
    hk844@cornell.edu,
    yusuke.narita@yale.edu, \\
    saito@hanjuku-kaso.com,
    Kei.Tateno@sony.com,
    Takuma.Udagawa@sony.com,
}

\usepackage{algorithm}
\usepackage{algorithmic}
\usepackage{bbding}
\usepackage{enumitem}
\usepackage{amsmath}
\usepackage{amsthm}
\usepackage{amsfonts}
\usepackage{mathtools}
\usepackage{threeparttable, booktabs, makecell, caption}
\usepackage{siunitx}

\renewcommand{\algorithmicrequire}{\textbf{Input:}}
\renewcommand{\algorithmicensure}{\textbf{Output:}}

\usepackage{multicol,multirow}
\usepackage{bbding}
\usepackage{comment}
\usepackage{bm}
\usepackage{color}

\definecolor{dkred}{rgb}{0.8,0,0}

\newtheorem{theorem}{Theorem}
\newcommand{\textbfit}[1]{\textbf{\textit{#1}}}

\newtheorem{assumption}{Assumption}[section]

\DeclareMathOperator*{\argmax}{arg\,max}

\newcommand{\mE}{\mathbb{E}}

\newcommand{\calD}{\mathcal{D}}
\newcommand{\calX}{\mathcal{X}}
\newcommand{\calA}{\mathcal{A}}

\begin{document}

\maketitle

\begin{abstract}
In many real recommender systems, novel items are added frequently over time. The importance of sufficiently presenting novel actions has widely been acknowledged for improving long-term user engagement. A recent work builds on Off-Policy Learning (OPL), which trains a policy from only logged data, however, the existing methods can be unsafe in the presence of novel actions. \textbfit{Our goal is to develop a framework to enforce exploration of novel actions with a guarantee for safety.}
To this end, we first develop \textit{Safe Off-Policy Policy Gradient} (Safe OPG), which is a model-free safe OPL method based on a high confidence off-policy evaluation. In our first experiment, we observe that Safe OPG almost always satisfies a safety requirement, even when existing methods violate it greatly.
However, the result also reveals that Safe OPG tends to be too conservative, suggesting a difficult tradeoff between guaranteeing safety and exploring novel actions.
To overcome this tradeoff, we also propose a novel framework called \textit{Deployment-Efficient Policy Learning for Safe User Exploration}, which leverages \textit{safety margin} and gradually relaxes safety regularization during multiple (not many) deployments. Our framework thus enables exploration of novel actions while guaranteeing safe implementation of recommender systems. 
\end{abstract}

\section{Introduction}
Many real-world interactive systems such as recommender and search evolve over time, with a growing number of feasible actions~\citep{chandak2020lifelong}.
For example, music streaming systems regularly refurbish their recommendation algorithms in response to the release of new songs.
Exploring such novel actions has widely been considered crucial, because they enhance diversity and novelty of recommendations, retain users' interests, reduce opportunity loss, and promote long-term engagement~\citep{chen2021values,kaminskas2016diversity}.
It is also valuable to provide a certain amount of exposure to novel actions from the fairness perspective~\citep{mehrotra2018towards,singh2018fairness}.
Indeed, enforcing diversity and novelty in recommender and search systems via \textit{user exploration}~\citep{chen2021values} have widely been recognized as crucial factors in Airbnb~\citep{abdool2020managing}, Youtube~\citep{chen2019top,chen2021values}, and Spotify~\citep{anderson2020algorithmic,mehrotra2021algorithmic}.

To discover new user interests, especially in the presence of novel actions, online bandits and reinforcement learning (RL) are widely used~\citep{chandak2020lifelong}. Based on the idea of \textit{optimism in the face of uncertainty}, online learning explores novel actions and learns their values through many online interactions~\citep{jaksch2010near}. For example, \citet{stamenkovic2022choosing} rely on a multi-objective online RL to achieve a balance between relevance and novelty of recommendations. While such online learning is effective in exploring novel actions, online exploration is often unsafe in practice because an online policy may choose low quality actions frequently in its initial learning phase. Moreover, frequent policy update is often infeasible due to its large implementation cost~\citep{matsushima2021deployment}.

Thus, a seminal work by~\citet{chen2021values} propose techniques for user exploration such as entropy regularization building on recent advances in Off-Policy Learning (OPL)~\citep{saito2021counterfactual,swaminathan2015batch}. OPL reduces the risk and cost of online learning by leveraging offline logged data collected by a logging policy and has successfully been applied to a static setup with a fixed set of actions~\citep{chen2019top,joachims2018deep,le2019batch,sachdeva2020off}. However, myopically applying OPL to the case with novel actions encounters critical safety issues. In particular, we empirically show that user exploration based on OPL~\citep{chen2021values} can become extremely unsafe in the sense that it significantly underperforms the logging policy, whereas the vital role of OPL is to improve and outperform the logging policy. Deploying such an unsafe policy should be avoided in practice, as a large performance drop may damage some important business metrics~\citep{gruson2019offline}.

Motivated by the aforementioned issues of online and offline learning regarding safety and user exploration, \textbfit{we aim at developing an easily implementable method that can explore novel actions, while guaranteeing safety at the same time}.
To this end, we first propose a modular safe OPL framework called \textit{Safe Off-Policy Policy Gradient} (Safe OPG), which solves a constrained optimization problem that requires a policy to exceed a predefined safety threshold (e.g., performance of the logging policy) with high probability. 
A fascinating feature of our Safe OPG over other safe off-policy learning methods~\citep{le2019batch,xu2021constraints} is that ours is totally model-free and does not need a model-based offline estimation of the safety constraint, which is likely to be unreliable in the presence of novel actions. Specifically, during its policy training phase, Safe OPG validates whether the policy meets the safety constraint based on High Confidence Off-Policy Evaluation (HCOPE)~\citep{thomas2015high}.
It then maximizes the policy value with a regularization for safety so that we can ensure that the performance is above a given threshold.
In the first semi-synthetic experiment, we show that Safe OPG mostly satisfies the safety constraint even in the existence of some novel actions, while the conventional OPL methods (combined with the user exploration technique of~\citet{chen2021values}) violate the safety requirement considerably. What is particularly problematic about the existing framework is that it sometimes substantially underperforms the logging policy, which is often the baseline policy we should improve upon. Our Safe OPG can automatically avoid this unsafe behavior of the existing framework for a range of logging policies. However, we also observe that Safe OPG barely chooses novel actions, because it becomes too conservative about those actions to \textit{guarantee} safety. These empirical observations imply a serious and inevitable tradeoff between guaranteeing safety and exploring novel actions.

To overcome this critical tradeoff between safety and novelty, we propose a new policy learning framework for safe exploration of novel actions, which we call \textit{\textbf{D}eployment \textbf{E}fficient \textbf{P}olicy Learning for \textbf{S}afe \textbf{U}ser \textbf{E}xploration (DEPSUE)} inspired by a framework of \textit{deployment-efficient policy learning} in offline RL~\citep{matsushima2021deployment}. Compared to the existing methods on top of conventional OPL, our DEPSUE is able to leverage \textit{safety margin} accumulated during previous deployments and relax safety regularization gradually to pursue novelty. As a result, DEPSUE avoids being overly conservative and enables \textit{safe exploration of novel actions} with deployment costs much lower than those of online learning. Additional experiments on semi-synthetic and real-world data demonstrate that DEPSUE can successfully explore novel actions without violating safety constraints and having costly policy deployments.

Our contributions are summarized as follows:
\begin{itemize}
    \item We formulate a new policy learning problem, which aims at exploring novel actions while guaranteeing safety.
    \item We develop Safe OPG, a model-free OPL method that guarantees safety even in the presence of novel actions.
    \item We empirically illustrate a tradeoff between guaranteeing safety and exploring novel actions for the first time.
    \item We propose a policy learning framework called DEPSUE, which is able to overcome the aforementioned tradeoff by leveraging past deployment results to gradually relax the safety regularization and to pursue novelty.
\end{itemize}

\section{OPL for User Exploration}
This section describes our policy learning problem in the contextual bandit setup and the objectives regarding novelty.

\subsection{Formulation and Objectives}
Let $x \in \calX \subseteq \mathbb{R}^{d_x}$ be a context vector (e.g., user demographics) and $r \in [0, r_{max}]$ be a reward (e.g., whether a user listens to a recommended song). 
Contexts and rewards are sampled from unknown probability distributions, $p(x)$ and $p(r|x,a)$, where $a \in \calA$ is a discrete action (e.g., a recommended song). 
We also observe action feature $e_a \in \mathcal{E} \subseteq \mathbb{R}^{d_a}$ (e.g., title, genre, and singer) for each action $a$. 
We use $q(x, e_a) := \mathbb{E}[r|x,e_a]$ to denote the expected reward function based on an assumption that the action feature retains information useful for inferring the reward function~\citep{saito2022off}. 
We call a function $\pi: \calX \rightarrow \Delta(\calA)$ a policy, where $\pi(a|x)$ is the probability of choosing action $a$ given context $x$.
We assume that a static logging policy $\pi_0: \calX \rightarrow \Delta(\calA_0)$ has already been deployed in the system. 
The logging policy can either be deterministic or stochastic, and we consider a realistic situation where $\pi_0$ can choose actions from only support action set $\calA_0 \subseteq \calA$. 
That is, there exist some novel actions $a \in \calA \setminus \calA_0$ in the system, which are unavailable to $\pi_0$. Exploring such novel actions is particularly important in recommender system applications, because it is known to enhance long-term user experience~\citep{anderson2020algorithmic,chen2021values,stamenkovic2022choosing}.
Presenting novel actions is also beneficial because it provides fair opportunities to all available actions~\citep{mehrotra2018towards,singh2018fairness} and stores valuable logged data for future policy learning~\citep{sachdeva2020off}. To satisfy these intrinsic needs for long-term platform success and fairness requirements, we aim at balancing the following (probably competing) objectives.
\begin{align*}
    \textit{Policy Value:} \,\, & V(\pi) := \mE_{p(x)\pi(a|x)p(r|x,a)}[r], \\
    \textit{Novelty:} \,\, & N(\pi) := \mathbb{E}_{p(x) \pi(a|x)} \left[ \mathbb{I} \{ a \in \calA \setminus \calA_0 \} \right],
\end{align*}
where $\mathbb{I}\{\cdot\}$ is the indicator function. \textit{Policy Value} is simply the expected performance of policy $\pi$~\citep{saito2020open}.
\textit{Novelty} measures how frequently a policy chooses novel actions~\citep{baeza1999modern}. 
Note that recent works on user exploration~\citep{chen2021values,stamenkovic2022choosing} pursue similar objectives to ours, however, these studies are unaware of the potential safety issue of user exploration. 
As we will show in the following sections, it is challenging to pursue novelty, while guaranteeing a certain level of policy value (which we call \textit{safety}). The main motivation of our work is thus to sufficiently explore novel actions while guaranteeing safety with much more reasonable and controllable implementation costs compared to online learning (via leveraging offline logged data).

\subsection{Existing Approaches and Safety Issues} \label{sec:existing_approach}
A promising approach for exploring novel actions is to train a policy by interacting with the online environment, widely known as \textit{online learning} (e.g., $\epsilon$-greedy~\citep{sutton2018reinforcement}, Thompson sampling~\citep{agrawal2013thompson}, and on-policy policy gradient~\citep{williams1992simple}).
In online learning, active exploration plays an important role in identifying high-quality novel actions. However, actively choosing uncertain actions may in turn exacerbate the policy performance if those actions are detrimental~\citep{kiyohara2021accelerating}. 
Moreover, online learning often requires substantial implementation costs, as the system has to repeatedly update the policy during its deployment. 
This makes online learning often infeasible in many practical applications~\citep{matsushima2021deployment}.

In response, Off-Policy Learning (OPL) has emerged as an alternative framework. The conventional OPL  maximizes the policy value $V(\pi)$ using only logged bandit data $\calD_0 := \{(x_i,a_i,r_i)\}_{i=1}^{n}$ collected by logging policy $\pi_0$~\citep{swaminathan2015batch}.\footnote{$(x,a,r) \overset{\mathrm{iid}}{\sim} p(x) \pi_0 (a | x) p(r | x, a)$.}
For example, Off-Policy Policy Gradient (OPG)~\citep{joachims2018deep,swaminathan2015batch} optimizes the policy parameter $\psi$ via gradient ascent: $\psi \leftarrow \psi + \eta \nabla_{\psi} \hat{V}(\pi_{\psi}; \calD_0)$ where $\eta$ is a learning rate. Even with some novel actions, we can estimate the policy gradient by leveraging the action feature $e_a$ and reward model $\hat{q}$, in a manner similar to the actor-critic method~\citep{konda2000actor}, as follows.
\begin{align}
    \nabla_{\psi} \hat{V}(\pi_{\psi}; \calD_0) := \mE_{n} \left[ \mE_{\pi_{\psi}(a|x_i)} [ \hat{q}(x_i, e_a)  \nabla_{\psi} \log \pi_{\psi}(a | x_i) ] \right], \label{eq:gradient}
\end{align}
where $\mE_{n} [\cdot]$ is the empirical expectation over $n$ data in $\calD_0$.

\citet{chen2021values} propose several modifications to the conventional objective of OPL with the aim of pursing diversity and novelty through user exploration. Specifically, they empirically demonstrate the advantages of entropy regularization, intrinsic motivation, and reward shaping in terms of long-term user engagement through a large-scale online experiment. In this paper, we consider only entropy regularization as a representative method, as it is easily applicable to our contextual bandit setup. Entropy regularization simply introduces the entropy term to the objective as:
\begin{align}
    \max_\pi \; (1 - \alpha) \, V(\pi) - \alpha \, \mE_{p(x)\pi(a|x)} [\log \pi(a|x)], \label{eq:entropy}
\end{align}
where $\alpha$ is a tradeoff hyper-parameter between policy value and entropy. The entropy term pushes a policy to be closer to a uniform random policy, promoting user exploration. As we will show, it is challenging to pursue novelty, while guaranteeing a certain level of policy value (which we call \textit{safety}). The main motivation of our work is thus to sufficiently explore novel actions while guaranteeing safety.

\section{Safe Off-Policy Learning} \label{sec:safe_opl}
We now present our first proposal called \textit{Safe Off-Policy Policy Gradient} (Safe OPG), which is able to guarantee a safe policy learning even when there exist some novel actions in the system.
We first formulate the safe OPL problem as the following constrained policy optimization problem.
\begin{align}
    \begin{array}{cl}
        \max_{\pi \in \Pi} \hat{V}(\pi; \calD_0) \quad \mathrm{s.t.}\;  \mathrm{Pr}(V(\pi) > C) \geq 1 - \delta
    \end{array} \label{eq:constrained_obj}
\end{align}
where the constraint requires a policy to outperform a predefined safety threshold $C$ with probability at least $1 - \delta$. 
One reasonable choice of $C$ is the policy value of the logging policy ($V(\pi_0)$) to avoid impractical behaviors of OPL that underperform the existing policy.

\begin{algorithm}[tb]
\caption{Safe Off-Policy Policy Gradient (Safe OPG)} \label{algo:safe_opg} 
  \begin{algorithmic}[1]
    \REQUIRE two folds of logged data $\calD_0^{\mathrm{(S1)}}$ and $\calD_0^{\mathrm{(S2)}}$, safety threshold $C$, learning rate for policy training $\eta_{\psi}$, learning rate for regularization $\eta_{\lambda}$, number of gradient updates $S$, confidence $\delta$
    \ENSURE learned policy $\pi_{\psi}$
    \STATE Initialize policy network parameter $\psi$
    \STATE $\lambda \leftarrow 0$
    \FOR{$s = 1, \ldots, S$}
      \STATE Update the policy parameter $\psi$ by Eq.~\eqref{eq:update_policy}
      \STATE Update the regularization parameter $\lambda$ by Eq.~\eqref{eq:update_lambda}
    \ENDFOR
  \end{algorithmic}
\end{algorithm}

To guarantee safety as in Eq.~\eqref{eq:constrained_obj}, we estimate a high probability lower bound of $V(\pi)$ based on HCOPE~\citep{thomas2015high}:
\begin{align}
    & \hat{V}_{-}(\pi; \calD_0) \nonumber \\
    &:= \max_{\tau \in \mathbb{R}_{+}} 
    \left( \mE_n \left[ z_i \right]
    - \sqrt{\frac{2 \log (2 / \delta)\mathbb{V}_n(z_i)}{n-1}}
    - \frac{7 \, z_{max} \log (2 / \delta)}{3 \, (n - 1)} \right), \label{eq:lower_bound}
\end{align}
where $z_i := \min \{ \pi(a_i | x_i) / \pi_0(a_i | x_i), \tau \} \cdot r_i$ and $z_{max} := \tau \cdot r_{max}$. $\mathbb{V}_n(\cdot)$ denotes the sample variance. Algorithm 1 of \citep{thomas2015confidence} describes the detailed procedure of HCOPE, which optimizes $\tau$ using 1/20 of the whole data and then applies the tuned value of $\tau$ to construct a lower bound based on the rest of the data.\footnote{The first term of the RHS of Eq.~\eqref{eq:lower_bound} (i.e., $\mE_n[z_i]$) is equivalent to an OPE estimator called Clipped Inverse Propensity Scoring~\citep{strehl2010learning}.}

Given the lower bound estimate of the policy value in Eq.~\eqref{eq:lower_bound}, our Safe OPG solves the following constrained optimization problem to maximize the policy value while satisfying the safety constraint with high probability.
\begin{align}
    \begin{array}{cl}
        \max_{\pi \in \Pi} \hat{V}(\pi; \calD_0^{\mathrm{(S1)}}) \quad
        \mathrm{s.t.}\; \hat{V}_{-}(\pi; \calD_0^{\mathrm{(S2)}}) > C
    \end{array} \label{eq:constrained_lower_bound}
\end{align}
where we use two different folds of data, $\calD_0^{\mathrm{(S1)}}$ and $\calD_0^{\mathrm{(S2)}}$, because Eq.~\eqref{eq:lower_bound} is only applicable when $\pi$ and $\calD_0^{\mathrm{(S2)}}$ are independent. To solve Eq.~\eqref{eq:constrained_lower_bound}, we cast it into the following minmax optimization.
\begin{align}
    \min_{\lambda \in \mathbb{R}_{+}} \max_{\pi \in \Pi} \, \mathcal{L}(\pi; \lambda, \calD_0^{(S1)}) := \hat{V}(\pi; \calD_0^{\mathrm{(S1)}}) + \lambda \mathcal{R}(\pi; \calD_0^{\mathrm{(S1)}}) 
    \label{eq:minmax}
\end{align}
where we use $\mathcal{R}(\pi; \calD_0) := \mE_{n} [ r_i \log \pi(a_i | x_i) ]$ as a regularization function. $\mathcal{R}(\cdot)$ enforces $\pi$ to imitate $\pi_0$ only when the observed action looks promising so as to improve the lower bound of the policy value and to help satisfy the safety constraint~\citep{ma2019imitation}. 
$\lambda$ is a parameter to control the magnitude of regularization.
Finding an appropriate $\lambda$ is the key to achieving an effective and safe policy learning, because it determines how conservative the policy should be.
To adaptively tune this parameter, we optimize $\pi$ and $\lambda$ via an iterative procedure. Specifically, Safe OPG first updates the policy parameter $\psi$ while fixing $\lambda$.
\begin{align}
    \psi \leftarrow \psi + \eta_{\psi} \nabla_{\psi} \mathcal{L}(\pi_{\psi}; \lambda, \calD_0^{\mathrm{(S1)}}), \label{eq:update_policy}
\end{align}
where we use only $\calD_0^{\mathrm{(S1)}}$ to estimate the policy gradient.
We then externally validate whether the safety constraint is met using $\calD_0^{\mathrm{(S2)}}$. If the safety constraint is violated, we should increase the value of $\lambda$ to impose a stronger regularization. Otherwise, we can decrease its value to relax the regularization. We achieve this adaptive tuning of $\lambda$ as follows.
\begin{align}
    \lambda \leftarrow \max \left \{ \lambda - \eta_{\lambda} \left( \hat{V}_{-}(\pi_{\psi}; \calD_0^{\mathrm{(S2)}}) - C \right), 0 \right \}, \label{eq:update_lambda}
\end{align}
where $\eta_{\lambda}$ is a learning rate. Algorithm~\ref{algo:safe_opg} provides the whole procedure of our Safe OPG.\footnote{In Appendix~\ref{app:proof}, we show that the algorithm converges to a local saddle point $(\pi^{\star}, \lambda^{\star})$ that maximizes Eq.~\eqref{eq:minmax}, and that $\pi^{\star}$ satisfies the safety constraint under some reasonable assumptions.} Compared to existing safe policy learning methods~\citep{le2019batch,xu2021constraints}, our Safe OPG is highly flexible and is compatible with a range of user exploration objectives such as entropy regularization. Moreover, Safe OPG is well suited to the setting with novel actions, because it does not rely on an accurate model-based approximation of the safety constraint, which is extremely difficult with offline data and in the presence of novel actions.

\section{Inevitable Tradeoff between safety and user exploration} \label{sec:experiment1}
This section conducts an experiment to evaluate whether Safe OPG ensures a safe OPL even in the presence of novel actions.
We also evaluate whether it is possible to \textit{safely} explore novel actions by merely combining Safe OPG and entropy regularization~\citep{chen2021values}.

\begin{figure*}[htbp]
  \begin{minipage}[b]{1.00\linewidth}
  \centering
  \begin{minipage}[b]{0.65\linewidth}
  \centering
  \includegraphics[width=0.70\linewidth]{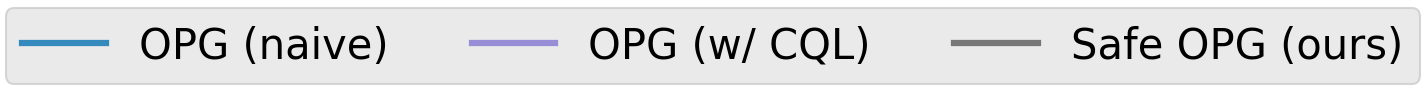}
  \end{minipage} \\
  \begin{minipage}[b]{0.98\linewidth}
  \centering
      \begin{minipage}[b]{0.33\linewidth}
        \centering
        \includegraphics[width=0.88\linewidth]{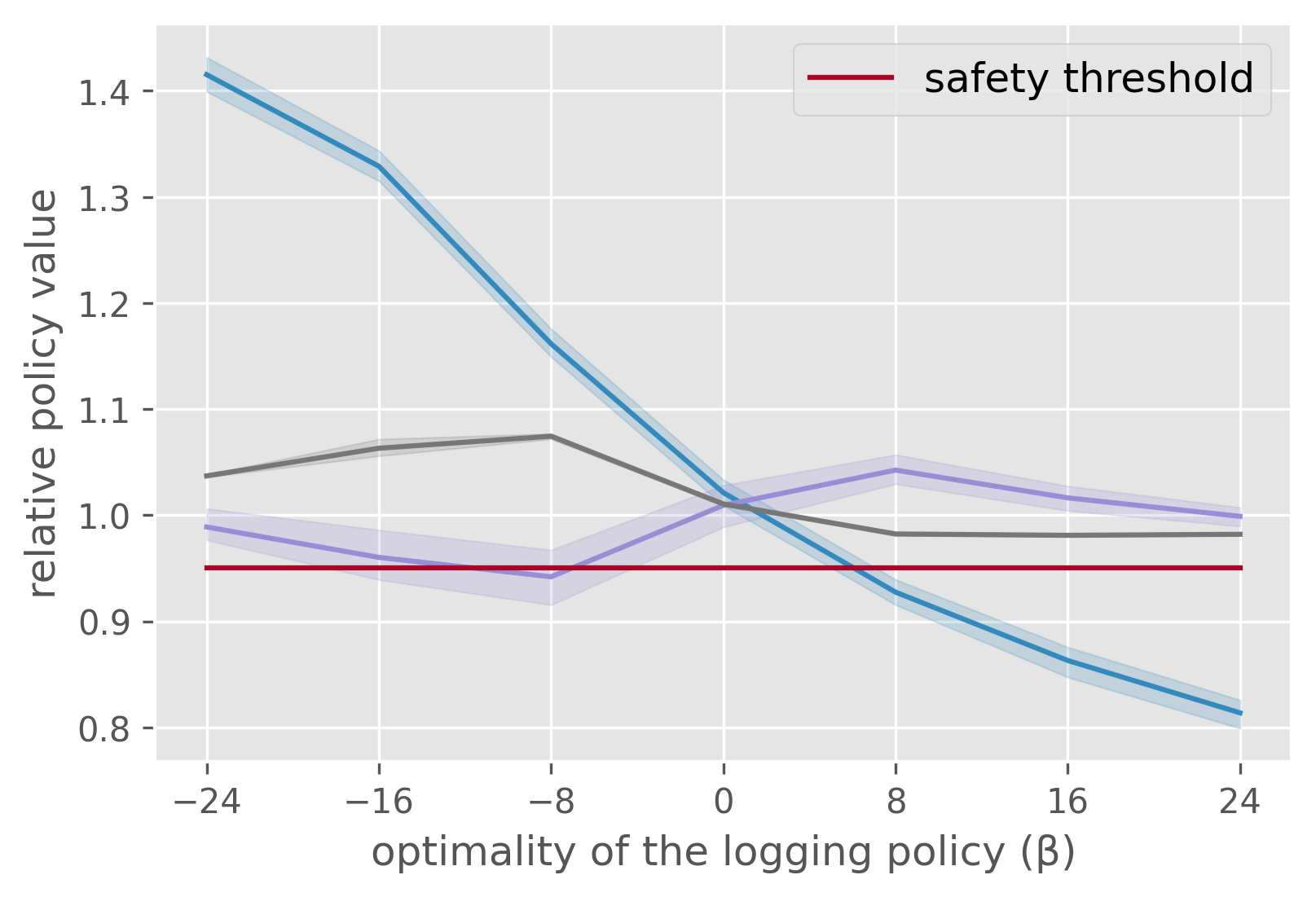}
        Relative Policy Value: $V(\pi) / V(\pi_0)$
      \end{minipage}
      \begin{minipage}[b]{0.33\linewidth}
        \centering
        \includegraphics[width=0.88\linewidth]{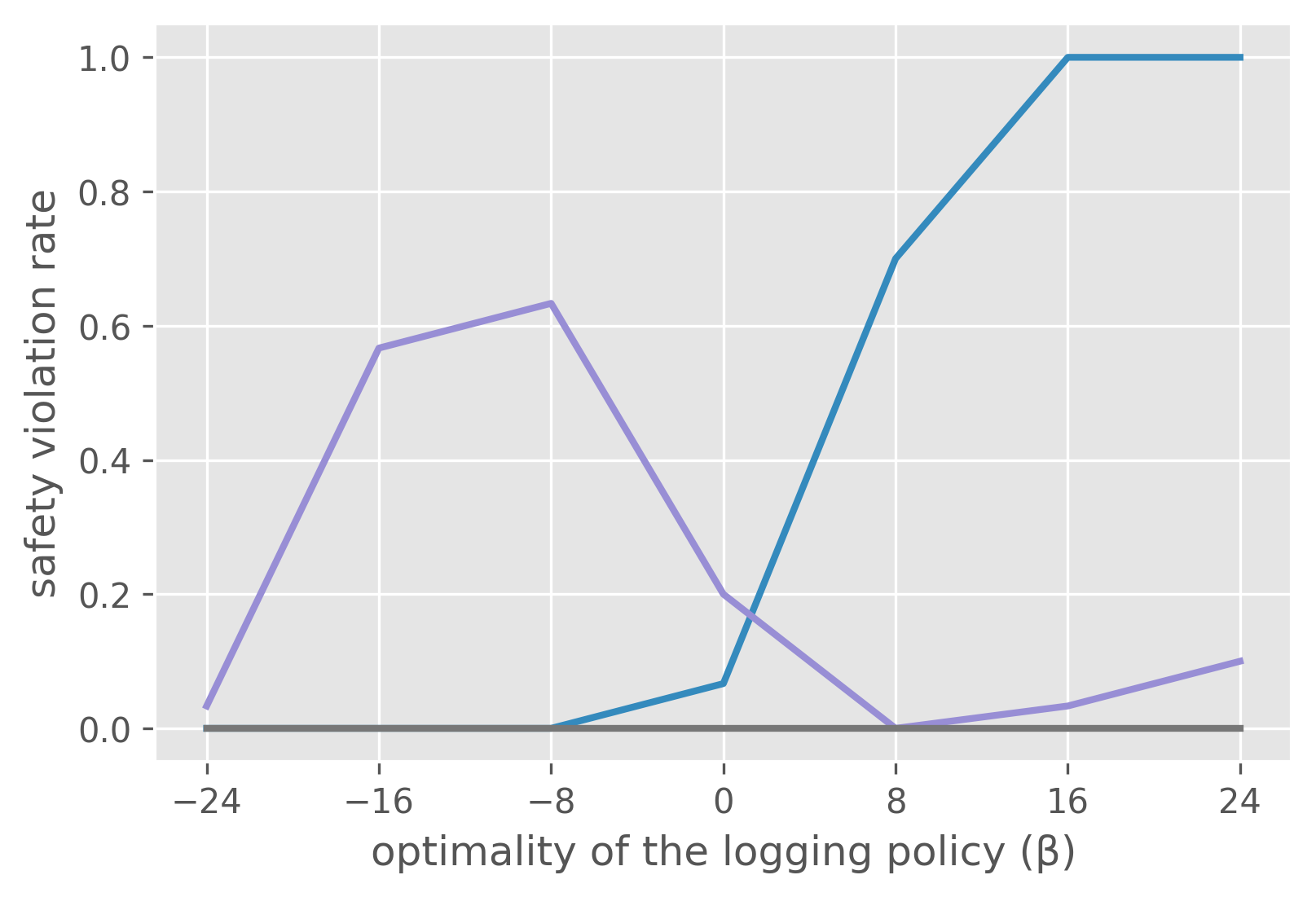}
        \\
        Safety Violation Rate: $P(V(\pi) < C)$
      \end{minipage}
      \begin{minipage}[b]{0.33\linewidth}
        \centering
        \includegraphics[width=0.88\linewidth]{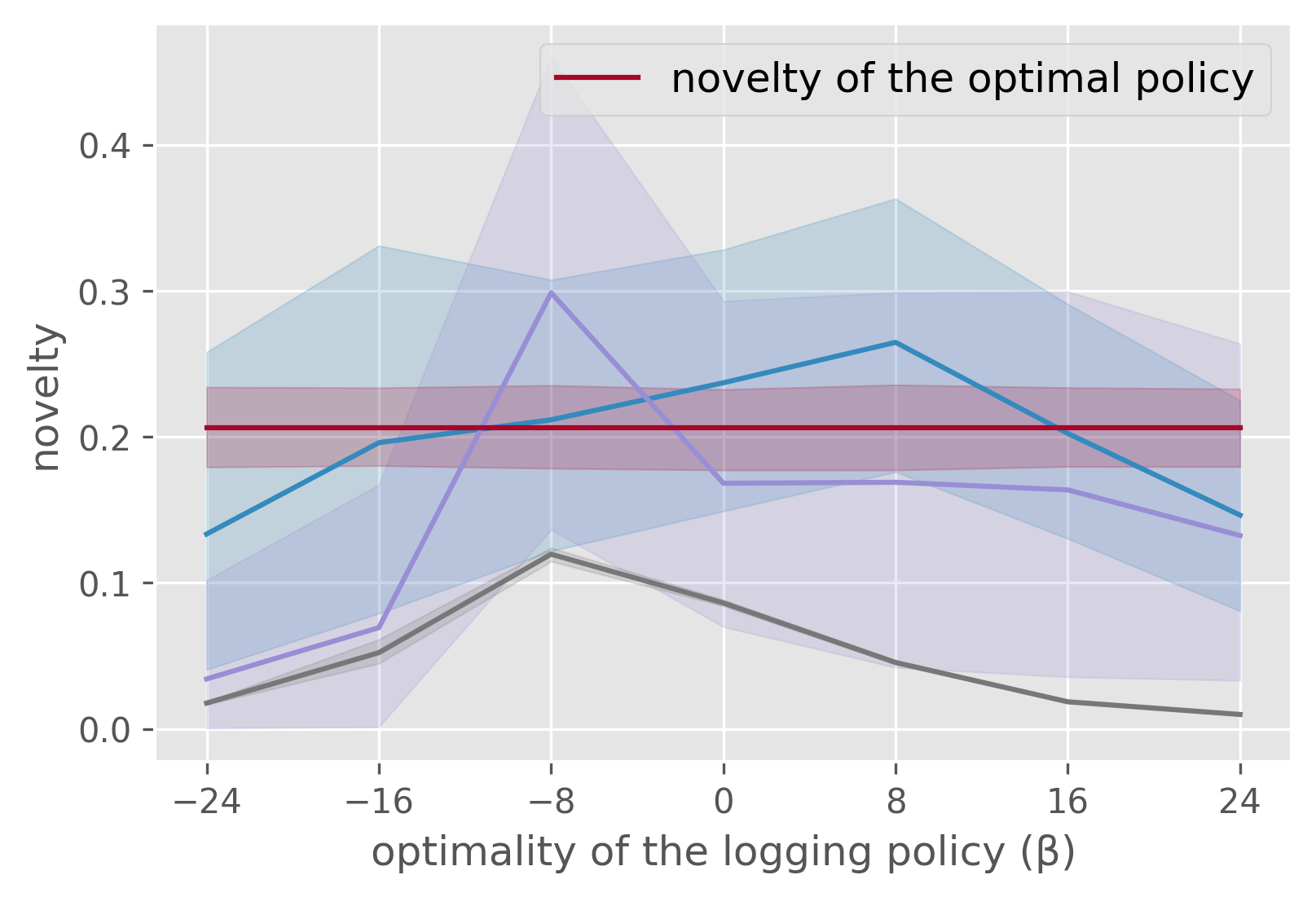}
        Novelty: $N(\pi)$
      \end{minipage}
      \vspace{-2mm}
      \caption{Comparing safety and novelty of Safe OPG, OPG (naive), and OPG (w/ CQL) under various logging policies; (a) relative policy value compared to that of the logging policy ($V(\pi) / V(\pi_0)$) with 95\% confidence interval, (b) probability of a policy violating the safety constraint ($P(V(\pi) < C)$) estimated with 30 simulation runs, and (c) novelty averaged over 30 simulation runs with 95\% confidence interval.}
      \label{fig:baseline_aggregate}
  \end{minipage}
  \end{minipage}
\end{figure*}

\subsection{Experiment Setup} \label{sec:setup1}

For this experiment, we build a semi-synthetic simulation environment using the \textit{MovieLens-1M} dataset~\citep{harper2015movielens}.\footnote{https://grouplens.org/datasets/movielens/1m/} 
The following specifies how we build the simulation environment. Further details (e.g., how we parameterize the synthetic reward function and policies) can be found in Appendix~\ref{app:experiment}.

\paragraph{Semi-Synthetic Data}
The MovieLens-1M dataset consists of approximately one million 5-scored movie ratings collected from 6,000 users and 4,000 movies. The dataset also provides some user demographic information (e.g., age and occupation) and categorical action features (e.g., title, year, and genre). Following~\citet{ma2020off}, we binarize the 5-scored rating data by regarding the highest score (5) as positive feedback and the other scores as negative feedback. 
We then synthesize the true expected reward function $q(x, e_a)$ by modeling the reward with a neural network using all available features in the original dataset.
Finally, we simulate a recommender environment by sampling a binary reward as $r \sim \operatorname{Bernoulli}(q(x, e_a))$ for each observed context $x$.
We then go on to generate the logged data $\calD_0$ based on the simulated reward function $q(x,a)$. We first randomly sample 800 movies as supported actions ($\calA_0$) and 200 movies as novel actions ($\calA \setminus \calA_0$).
Among the users who have rated at least one of the above actions, we randomly sample $80\%$ of users as training users and the others as testing users.
We also define a logging policy based on the simulated reward function $q(x,a)$ as:
\begin{align*}
    \pi_0(a | x) := \mathbb{I}\{a \in \calA_0\} \frac{\exp(\beta \cdot \mathrm{logit}(q(x, a)))}{\sum_{a' \in \calA_0}\exp(\beta \cdot \mathrm{logit}(q(x, a')))}
\end{align*}
where $\mathrm{logit}(z) := \log(z / (1 - z))$ and $\mathbb{I}\{\cdot\}$ is the indicator function. $\beta \in \{-24, -16, -8, 0, 8, 16, 24\}$ is the inverse temperature parameter of the softmax function, which determines the optimality and stochasticity of $\pi_0$. A large positive value of $\beta$ leads to a near-optimal and near-deterministic policy, while a negative value leads to a policy worse than uniform random.
Finally, we iterate the following process to generate $\calD_0$ consisting of 500,000 samples of observations.
\begin{enumerate}
    \item uniformly pick a context $x$ from the training users
    \item given $x$, sample an action based on $\pi_0(a | x)$
    \item given $x$ and $e_a$, sample a reward based on $q(x, e_a)$ where the action feature $e_a$ are recorded in the original dataset
\end{enumerate}

\paragraph{Compared Methods}
We compare our Safe OPG against OPG (naive) and OPG (w/ CQL). As we described in the previous section, OPG (naive) uses a direct estimate of $\hat{q}$ to estimate the policy gradient while OPG (w/ CQL) is based on a conservative estimate~\citep{kumar2020conservative}.
Safe OPG uses the same $\hat{q}$ as OPG (naive), but ensures that a lower bound of the policy value is above the safety threshold, which we set here as $C = 0.95 \cdot \hat{V}_{\mathrm{on}}(\pi_0) = 0.95 \cdot \mE_n [r_i]$. Note that we divide the logged data evenly into $\calD_0^{\mathrm{(S1)}}$ and $\calD_0^{\mathrm{(S2)}}$ to implement Safe OPG. All the methods above enforce user exploration via entropy regularization where we set $\alpha = 0.1$ in Eq.~\eqref{eq:entropy}.

\subsection{Results and Discussions} \label{sec:result1}
Figure~\ref{fig:baseline_aggregate} compares Safe OPG, OPG (naive), and OPG (w/ CQL) with varying optimality ($\beta$) of $\pi_0$. 
Specifically, Figure~\ref{fig:baseline_aggregate} (a) reports the policy value relative to that of the logging policy ($V(\pi) / V(\pi_0)$). Figure~\ref{fig:baseline_aggregate} (b) reports the probability a policy violating the safety constraint ($P(V(\pi) < C)$). Figure~\ref{fig:baseline_aggregate} (c) shows novelty ($N(\pi)$). As a reference, Figure~\ref{fig:baseline_aggregate} (c) includes the novelty of the optimal policy: $\pi^{\ast}(a | x) := \mathbb{I} \{a=\argmax_{a' \in \mathcal{A}} q(x, a')\}$. Note that we run the simulation with 30 different random seeds for each logging policy and calculate the metrics on 500,000 test samples. 

\paragraph{\textbf{(i) Can the OPL methods avoid the unsafe behavior under various logging policies?}}
Here, we analyze the performance and safety of different OPL methods.
First, we observe that OPG (naive) violates the safety constraint considerably when $\pi_0$ is near-optimal (i.e., $\beta\ge 8$). This is because $\hat{q}$ suffers from distributional shift in $\calD_0$, which results in a serious overestimation of the value of low quality actions.\footnote{We show the overestimation issue of the reward model $\hat{q}$ in detail in Appendix~\ref{app:experiment}.}
In addition, Figure~\ref{fig:baseline_aggregate} (b) indicates that OPG (w/ CQL) frequently violates the safety constraint when $\beta=-16,-8,0$. This is because CQL becomes overconfident about the value of sub-optimal supported actions if there are novel actions. These observations suggest that conventional methods become impractical in the presence of novel actions, potentially deploying unacceptably poor-performing policies.

In contrast, Figure~\ref{fig:baseline_aggregate} (b) shows that our Safe OPG never violates the safety constraint on a range of logging policies, indicating that Safe OPG can ensure safety using only the logged data and even in the presence of novel actions where the baseline methods can become very unsafe.

\paragraph{\textbf{(ii) Can the OPL methods choose and explore novel actions?}}
The previous section demonstrates that Safe OPG can avoid deploying catastrophic policies without any validation scheme (i.e., inherently ensuring safety during policy training). However, we also observe some critical drawbacks of Safe OPG in terms of its novelty.
Figure~\ref{fig:baseline_aggregate} (c) reveals that Safe OPG hardly chooses novel actions even though $\pi_0$ performs far worse than uniform random, while OPG (naive) and OPG (w/ CQL) choose novel actions much more frequently, leading to a better novelty metric. We attribute this phenomenon to the underestimation of the value of novel actions. Specifically, the regularization function of Safe OPG cannot evaluate novel actions due to the lack of support. This makes Safe OPG too conservative, leading to a substantial opportunity loss and unfair treatment of the novel actions.

\begin{table*}[htb]
\large
\centering
\caption{Conceptual comparison of DEPSUE with the online and offline policy learning framework} \label{tab:concept}
\vspace{-2mm}
\def\arraystretch{1.2}
\scalebox{0.8}{
\begin{tabular}{c|c|c|c}
\toprule
 & \textbf{Offline} & 
\textbf{DEPSUE (on top of Safe OPG)} & 
\textbf{Online} 
\\\midrule \midrule
safety 
& 
\begin{minipage}[b]{0.30\linewidth}
\centering
unsafe with novel actions
\end{minipage}
& 
\begin{minipage}[b]{0.30\linewidth}
\centering
guaranteed
\end{minipage}
&
\begin{minipage}[b]{0.30\linewidth}
\centering
unsafe in the initial learning
\end{minipage}
\\
novelty & insufficient & conservative, but gradually improves & sufficient \\
\# of deployments & only once ($K=1$), low cost & a few times ($K$), controllable cost & many times, high cost \\
\bottomrule
\end{tabular}
}
\end{table*}

\begin{figure*}[htbp]
  \centering
  \begin{minipage}[b]{0.60\linewidth}
  \centering
  \includegraphics[width=0.75\linewidth]{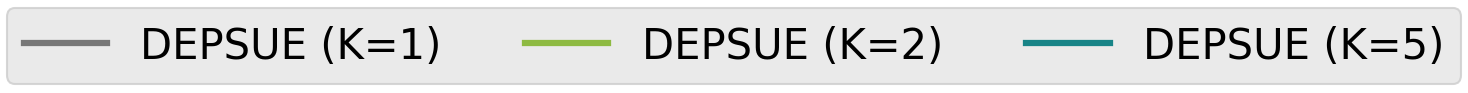}
  \end{minipage} \\
  \begin{minipage}[b]{0.80\linewidth}
  \centering
      \begin{minipage}[b]{0.38\linewidth}
        \centering
        \includegraphics[width=0.95\linewidth]{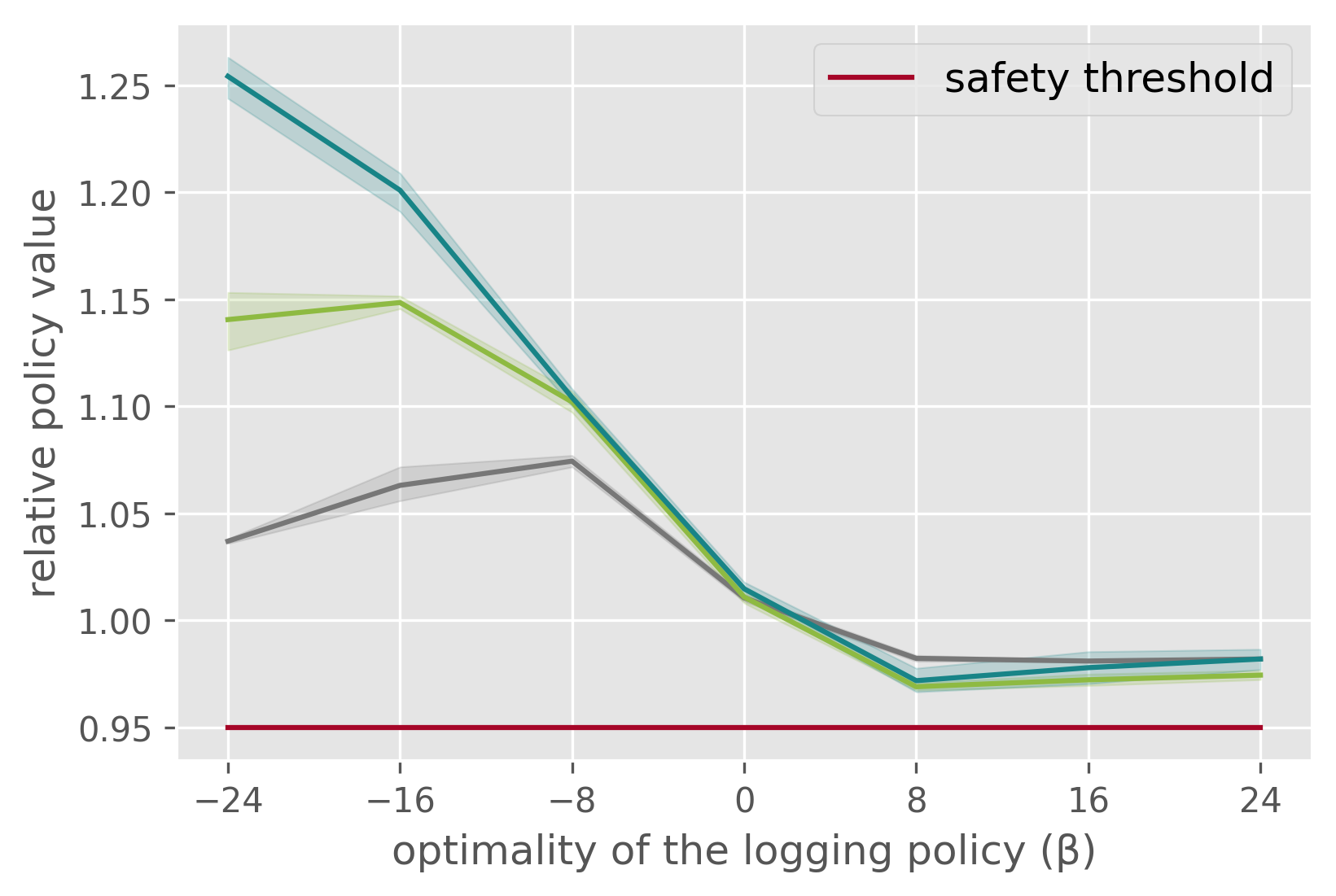}
        \\
        \textcolor{white}{xx} Relative Policy Value
      \end{minipage}
      \begin{minipage}[b]{0.38\linewidth}
        \centering
        \includegraphics[width=0.95\linewidth]{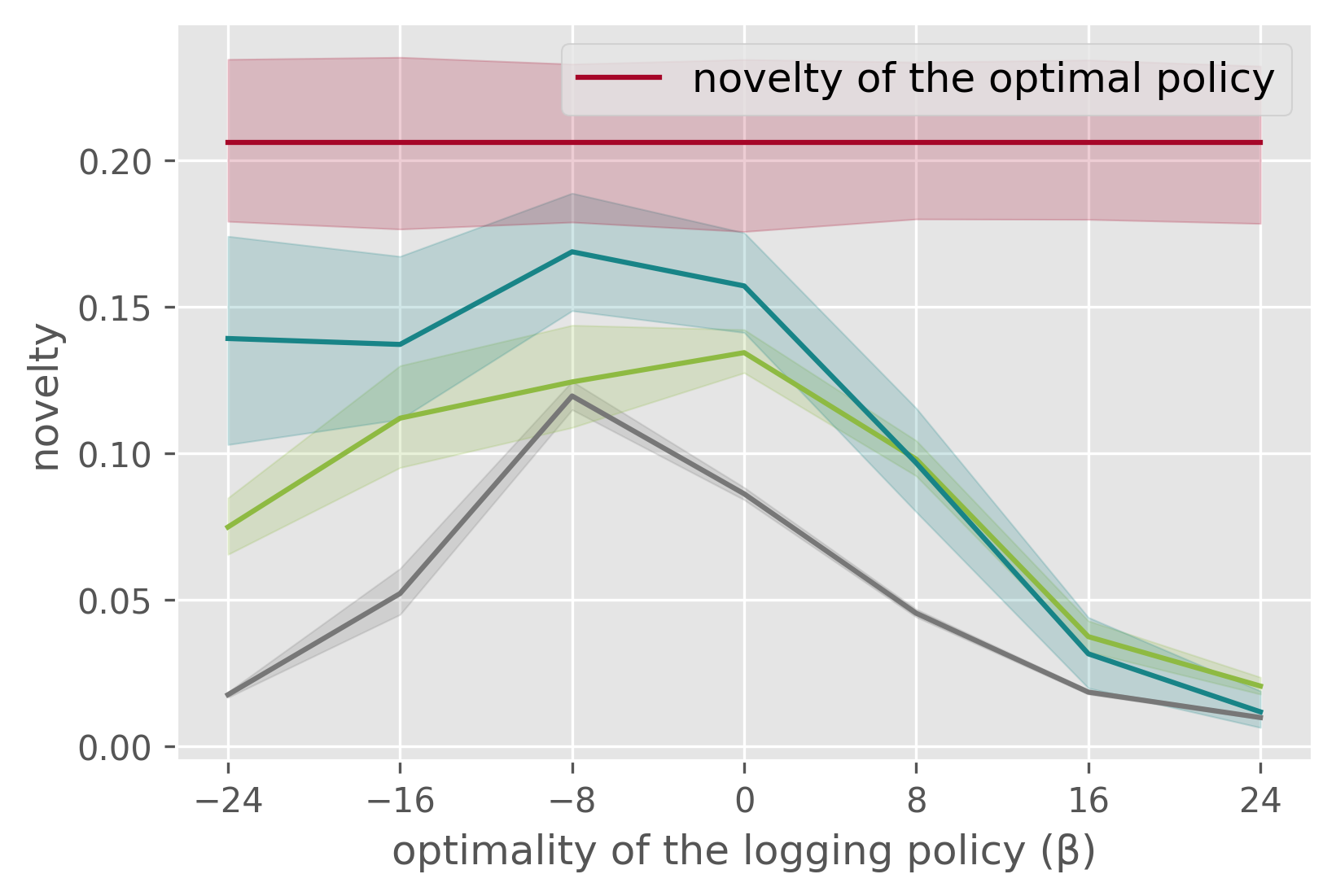}
        \\
        \textcolor{white}{xxx} Novelty
        \label{fig:one_stage}
      \end{minipage}
  \end{minipage}
  \caption{Evaluating safety and novelty of DEPSUE under various logging policies; (a) relative policy value compared to that of the logging policy ($V(\pi) / V(\pi_0)$) and (b) novelty, averaged over 30 simulation runs with 95\% confidence interval.}
  \label{fig:safe_aggregate}
\end{figure*}

\paragraph{\textbf{Summary of the tradeoff between safety and exploration}}
We have seen that the conventional OPL methods combined with entropy regularization can choose novel actions, but sometimes become unsafe and considerably underperform the logging policy. In contrast, Safe OPG succeeds in ensuring safety from only the logged data, but in turn hesitates to choose novel actions even when $\pi_0$ performs poorly. This tradeoff is attributed to the value estimation of the novel actions – overestimation easily leads to an unsafe policy, while underestimation is likely to result in a too conservative policy. Adjusting the balance between these competing motivations is, however, often challenging, as there is a large uncertainty in the value of novel actions in nature. In particular, guaranteeing safety is often demanding in the presence of novel actions, merely leading to an overly conservative policy that rarely chooses novel actions. This observation motivates us to further develop an easily implementable framework to enable a safe exploration of novel actions.

\section{Deployment-Efficient and Safe Exploration of Novel Actions} \label{sec:multi_stage}
The previous section sheds light on the difficulty in pursuing safety and exploration of novel actions at the same time. To overcome this challenging tradeoff, we now propose a new framework for safe exploration of novel actions called \textit{DEPSUE}. Motivated by \textit{deployment-efficient} policy learning in offline RL~\citep{matsushima2021deployment}, DEPSUE addresses the aforementioned tradeoff by gradually relaxing the regularization for safety leveraging what is called the \textit{safety margin} accumulated during multiple deployments. For example, if the policies in earlier stages exceed the safety threshold by a large margin, we should be able to relax the regularization and enforce a stronger user exploration in later deployments while still guaranteeing safety. Building on this intuition, our DEPSUE achieves safe user exploration while updating the policy $K$ times during its deployment. The choice of $K$ depends on the policy deployment cost and should appropriately be set by practitioners. Nonetheless, we will show that DEPSUE is highly effective in terms of novelty as well as the policy value, even if $K$ is quite small (say 2 to 5), making it much more easily implementable compared to online learning (from both safety and cost perspectives). The following provides the policy training procedure of DEPSUE for the $k$-th deployment policy $\pi_k$.
\begin{align}
    \max_{\pi_k \in \Pi} & \quad F_k(\pi_k; \calD_{k-1}^{\mathrm{(S1)}}) \label{eq:depsue_obj} \\
    \mathrm{s.t.} & \left\{ \begin{array}{ll}
                \hat{V}_{-}(\pi_k; \calD_0^{\mathrm{(S2)}}) > C & (k=1) \nonumber \\
                \hat{V}_{-}(\pi_k; \calD_{k-1}^{\mathrm{(S2)}}) + \sum_{k'=1}^{k-1} \hat{V}_{\mathrm{on}}(\pi_{k'}; \calD_{k'})  > kC & (k\ge2)
                \end{array} \right. \nonumber
\end{align}
where $\calD_k$ is the logged data collected by $\pi_k$ consisting of $m_k$ interactions. $\hat{V}_{\mathrm{on}}(\pi_k; \calD_k) := \mE_{m_k}[r_i]$ is the \textit{on-policy} policy value estimate of $\pi_k$. $F_k(\cdot)$ is virtually any objective for user exploration, such as the policy value plus entropy regularization as in Eq.~\eqref{eq:entropy}.

A key trick here is that DEPSUE trains $\pi_k$ so that the average policy value up to the $k$-th round is guaranteed to exceed the safety threshold (i.e., $\sum_{k'=1}^k V(\pi_k) \ge kC$) with high probability. This equivalently means that we impose $V(\pi_k) \ge kC - \sum_{k'=1}^{k-1} V(\pi_k)$ as a safety constraint for $\pi_k$, which can be considerably less restrictive if we have some positive \textit{safety margin} up until $k$-th deployment, i.e., $\sum_{k'=1}^{k-1} V(\pi_k) > (k-1)C$. In this way, we relax the regularization to explore novel actions more when possible while satisfying a safety requirement at \textit{every} deployment from $k=1$ to $k=K$. Table~\ref{tab:concept} compares the concept of DEPSUE with conventional offline and online learning frameworks.

\begin{table*}[h]
\begin{tabular}{cc}
\begin{minipage}{1.0\textwidth}
\large
\centering
\caption{Results of the Real-World Experiment on Wiki-31k}
\vspace{-2mm}
\def\arraystretch{1.2}
\scalebox{0.70}{
\begin{tabular}{cccccccccccccc}
\toprule
 \textbf{logging policies} && \multicolumn{3}{c}{\textbf{poor} ($\epsilon=0.8$)} && \multicolumn{3}{c}{\textbf{medium} ($\epsilon=0.5$)} && \multicolumn{3}{c}{\textbf{near-optimal} ($\epsilon=0.2$)} \\ \cmidrule{1-1} \cmidrule{3-5} \cmidrule{7-9} \cmidrule{11-13}
\textbf{OPL methods} && \textbf{Violation} & \textbf{Policy Value} & \textbf{Novelty} && \textbf{Violation} & \textbf{Policy Value} & \textbf{Novelty} && \textbf{Violation} & \textbf{Policy Value} & \textbf{Novelty} \\\midrule \midrule
\textbf{OPG} (naive)
&& \textcolor{blue}{$\mathbf{30/30}^{\dagger}$} & 0.837 & 0.167
&& \textcolor{blue}{$\mathbf{30/30}^{\dagger}$} & 0.674 & 0.123
&& \textcolor{blue}{$\mathbf{30/30}^{\dagger}$} & 0.563 & 0.131 \\
\textbf{OPG} (w/ CQL) 
&& \textcolor{blue}{$\mathbf{30/30}^{\dagger}$} & 0.843 & 0.284
&& \textcolor{blue}{$\mathbf{30/30}^{\dagger}$} & 0.689 & 0.254
&& \textcolor{blue}{$\mathbf{30/30}^{\dagger}$} & 0.592 & 0.301 \\ 
\textbf{Naive Safe Exploration}
&& 0/30 & 0.992 & 0.050
&& 0/30 & 0.984 & 0.050
&& 0/30 & 0.978 & \textcolor{dkred}{$\mathbf{0.050}^{\ast}$} \\
\midrule
\textbf{DEPSUE} ($K=1$) 
&& 0/30 & \textcolor{dkred}{$\mathbf{1.094}^{\ast}$} & 0.091
&& 0/30 & \textcolor{dkred}{$\mathbf{1.099}^{\ast}$} & 0.032
&& 0/30 & \textcolor{dkred}{$\mathbf{1.038}^{\ast}$} & 0.011 \\
\textbf{DEPSUE} ($K=2$) 
&& 0/30 & 1.057 & 0.100
&& 0/30 & 1.087 & 0.036
&& 0/30  & 1.030 & 0.011 \\ 
\textbf{DEPSUE} ($K=5$) 
&& 0/30 & 1.032 & \textcolor{dkred}{$\mathbf{0.108}^{\ast}$}
&& 0/30 & 1.010 & \textcolor{dkred}{$\mathbf{0.067}^{\ast}$}
&& 2/30 & 0.968 & 0.024 \\ 
\bottomrule
\end{tabular}
}
\vskip 0.1in
\raggedright
\fontsize{8.5pt}{8.5pt}\selectfont \textit{Note}:
A lower value is better for violation, while a higher value is better for the other metrics. 
The \textcolor{dkred}{$\mathbf{red^{\ast}}$} fonts indicate the best policy among those satisfying the safety constraint. The \textcolor{blue}{$\mathbf{blue^{\dagger}}$} fonts indicate policies that violate the safety constraint more frequently than $\delta$.
\label{tab:wiki}
\end{minipage}
\end{tabular}
\end{table*}

\section{Semi-Synthetic Experiment} \label{sec:experiment2}
Here, we demonstrate how DEPSUE handles the tradeoff between safety and exploration of novel actions using the same semi-synthetic setting as in Section~\ref{sec:experiment1}.
We evaluate DEPSUE with $K=1,2,5$, where $K=1$ is equivalent to Safe OPG in the previous experiment. 
Note that we (randomly) observe only $|\calD_k| = |\calD| /K$ observations at each deployment phase for a fair comparison.

\subsection{Results} \label{sec:result2}
Figure~\ref{fig:safe_aggregate} reports (a) relative policy value and (b) novelty of DEPSUE.
The results clearly demonstrate that DEPSUE has the promising potential to solve the tradeoff between safety and novelty observed in the previous experiment (Section~\ref{sec:experiment1}). First, we observe that DEPSUE gradually improves novelty as $K$ increases, while always satisfying the safety constraint.
Moreover, the value of DEPSUE is much higher than OPG (w/ CQL) when $\beta < 0$. 
These results suggest that DEPSUE successfully presents novel actions and enhances novelty when the logging policy is bad, while ensuring safety under near-optimal logging policies. It is remarkable to see that our framework automatically explores novel actions only when possible, even though the optimality of the logging policy is totally unknown, and we argue that DEPSUE (combined with Safe OPG) is the only policy learning framework that enables sufficient exploration of novel actions, guarantees safety in a range of logging policies, and is much more easily implementable compared to online learning. 

\section{Real-World Experiment}
This section further evaluates DEPSUE using the extreme classification dataset called Wiki10-31K~\citep{bhatia16,yu2022pecos}\footnote{https://github.com/amzn/pecos/tree/mainline/examples/spmm\#data-statistics}. Following previous works on OPL~\citep{dudik2014doubly, farajtabar2018more}, we convert the classification data into bandit data with partially observable rewards.

\subsection{Setup}
Wiki10-31K consists of approximately 20K documents with 31K labels. Each document is associated with multiple positive labels. Since some of the labels are positive for most of the documents, we first remove the labels that are positive for more than 1K documents. Then, we extract 10K frequent labels as actions, 8K of which constitute supported actions $\calA_0$ and the rest is deemed novel. We use only the documents that have at least one positive label among $\calA$. To preprocess the document features (contexts), we first obtain 578 dimensional embeddings from a pretrained-bert language model~\citep{reimers2019sentence}.\footnote{``bert-base-nli-mean-tokens'' with max\_seq\_length=512} Then, we compress the embeddings into 20-dimensional context $x$ via PCA~\citep{pearson1901liii}.
Among them, 80\% of the documents are used only for policy training, while the others are used for testing. Finally, we sample binary rewards from the Bernoulli distribution, where we let $q(x, a) = 0.8$ for the actions with a positive label and $q(x, a) = 0.2$ for negatives. We should note that this setting is extremely challenging compared to the previous semi-synthetic setting due to the scarcity of positive labels.

To obtain a logging policy $\pi_0$ and action embeddings $e_a$, we first train a supervised classifier with the cross entropy loss. Following~\citet{ma2020off}, we use a two-tower model as the classifier, which first individually encodes the context and action into 20-dimensional vectors and produces $\mathrm{logit}(x, a)$ from their inner product. Note that the obtained action embeddings $e_a$ are used for training the reward model $\hat{q}$, but not for policy training.
After learning a classifier, we define the following epsilon-greedy-style logging policy based on the estimated class label $\hat{a} \, (\in \calA_0)$ as follows.
\begin{align*}
    \pi_0(x, a) := 
    \begin{cases}
        (1 - \epsilon) \, \mathbb{I}\{a = \hat{a}\} + \epsilon / |\calA_0| & (a \in \calA_0) \\
        0 & (a \in \calA \setminus \calA_0)
    \end{cases}
\end{align*}
where $\epsilon$ is the degree of exploration. We set $\epsilon=0.8, 0.5, 0.2$, and call the corresponding logging policies \textit{poor}, \textit{medium}, and \textit{near-optimal}, respectively.\footnote{All logging policies perform better than uniform random.} Since positive labels are scarce, we clip $\hat{V}_{\mathrm{on}}(\pi_k)$ in calculating the safety margin in DEPSUE by the maximum value of $1.05 \cdot V(\pi_0)$. The other settings follow Section~\ref{sec:experiment1}, except that we additionally compare \textit{Naive Safe Exploration}, which follows $\pi_0$ with probability $0.95$ to ensure safety and uniformly explores novel actions with probability $0.05$.

\subsection{Results}
Table~\ref{tab:wiki} compares the violation (frequency of each policy violating the safety constraint), policy value, and novelty of OPG (naive), OPG (CQL), and DEPSUE ($K=1,2,5$). These metrics are obtained by running the policy learning experiment 30 times with different random seeds. First, we observe that the baseline OPGs violate the safety constraint in every simulation run (30/30), indicating that guaranteeing safety by a model-based approach is almost impossible in the presence of novel actions.

In contrast, DEPSUE satisfies the safety constraint perfectly. Moreover, DEPSUE enhances its novelty gradually with increasing $K$, suggesting that DEPSUE enables a steerable tradeoff among safety, novelty, and implementation cost. Specifically, we observe that DEPSUE starts from a conservative policy and relaxes the regularization to explore novel actions only when there is enough safety margin. Furthermore, DEPSUE explores more novel actions than naive safe exploration in poor and medium policy cases. DEPSUE is also better than naive safe exploration in terms of policy value. This suggests that adaptively controlling the safety regularization is the key to achieving an efficient safe exploration of novel actions.

\section{Conclusion and Future Work}
We have studied and overcome the fundamental tradeoff in safety and exploration of novel actions.
Our central motivation comes from the empirical observation that the existing framework for user exploration can substantially underperform the logging policy in the presence of novel actions. To guarantee safety in a modular way, we first develop Safe OPG, which uses only offline logged data to ensure that a lower bound of the policy performance is above a given threshold even when novel actions exist. In the first set of experiments, we showed that Safe OPG satisfies the safety constraint. However, we also identified an inevitable tradeoff between safety and exploration. To overcome this tradeoff, we also propose a new high-level policy learning framework called DEPSUE to enable safe exploration of novel actions. By controlling the regularization for safety during multiple deployments, DEPSUE gradually enhances novelty without ignoring the safety requirement. Additional empirical results demonstrate that DEPSUE addresses the difficult tradeoff of safety and exploration in a practically implementable way.

An interesting direction for future work is to consider a more elaborate use of action features in our framework. 
For example, it would be valuable to utilize action features in deriving a tighter lower bound of the policy value to avoid being overly conservative in ensuring safety.
Moreover, it would be interesting to develop a more efficient strategy that leverages action features, beyond the mere entropy regularization, for exploring novel actions.

\bibliography{main.bbl}
\clearpage
\appendix

\begin{figure*}
    \centering
    \includegraphics[width=0.80\linewidth]{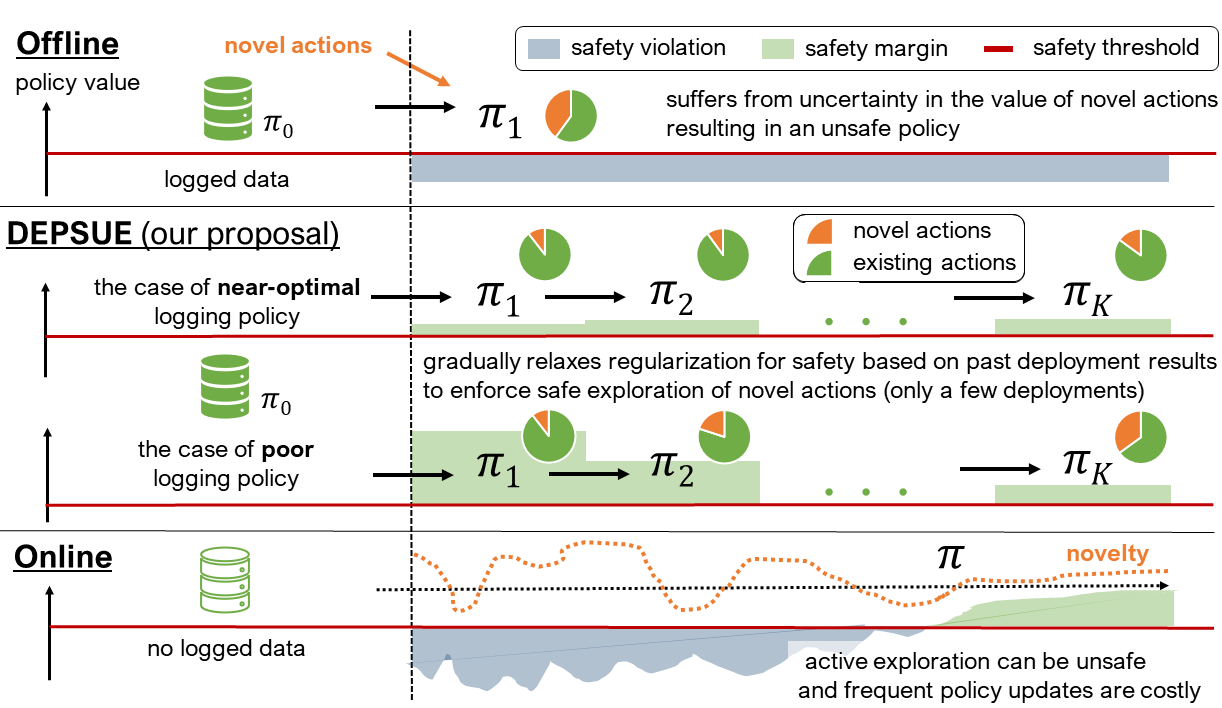}
    \caption{Overview of Safe and Deployment-Efficient Policy Learning for User Exploration (DEPSUE), which gradually relaxes the safety regularization to sufficiently explore novel actions while guaranteeing safety.}
    \label{fig:depsue}
\end{figure*}

\section{Related Work} \label{app:related}

\paragraph{\textbf{Strategic and User Exploration}}
The tradeoff between exploration and exploitation has been a central topic of bandits and RL. The most well-known techniques for the \textit{strategic} exploration include Upper Confidence Bound~\citep{lai1985asymptotically} and Thompson Sampling~\citep{chapelle2011empirical}. The motivation for these methods for strategic exploration is to collect information about actions that the algorithms are less certain about. In this way, bandits and RL policies gain better model quality at a later time to achieve a better regret than greedy exploitation.

This classical strategic exploration is believed to cost user experience in the short term. However, another aspect of exploration, referred to as \textit{user exploration}, has recently been established in the context of recommender systems to discover novel user interests~\citep{chen2021values}.
Indeed, it has become increasingly recognized that presenting a novel and diverse set of actions via user exploration is valuable in terms of long-term objectives in recommender systems, as it fosters diversity, novelty, and fairness among items~\citep{kaminskas2016diversity, lam2008addressing, singh2018fairness, wang2021fairness}. Note here that diversity is applied over a set
of items and refers to how different these items are. In contrast, novelty refers to how novel an item is for a user, given what the user has already seen.
A large volume of work has indicated the importance of enforcing diversity and novelty on real online platforms~\citep{abdool2020managing,anderson2020algorithmic,chen2021values,mehrotra2021algorithmic}, proposed novel evaluation metrics incorporating accuracy, diversity, and novelty~\citep{abdollahpouri2021user,chandar2013preference,clarke2008novelty,parapar2021towards}, developed novel algorithms for user exploration~\citep{li2020cascading,liu2020diversified,parapar2021diverse,stamenkovic2022choosing,wilhelm2018practical}.
In particular, \citet{stamenkovic2022choosing} propose a multi-objective online RL framework called SMORL with the aim of optimizing the accuracy, diversity, and novelty of recommendations simultaneously.
In contrast, \citet{chen2021values} formulate user exploration on top of OPL and incorporate entropy regularization, intrinstic motivaton, and reward shaping to the typical objective. \citet{chen2021values} also empirically verify that user exploration in recommender systems improves long-term user satisfaction by pursuing novelty and serendipity.
However, as we argue in Section~\ref{sec:experiment1}, these previous works on policy learning for user exploration do not take \textit{safety} into account. Although safe policy learning has been studied as we will discuss, improving diversity and novelty while guaranteeing safety \textit{at the same time} still remains less-explored despite its practical relevance. We are the first to study and highlight the critical tradeoff between enhancing user exploration and guaranteeing safety.

Note that there is a line of work studying \textit{safe exploration} in an online learning setup~\citep{bura2021safe, pacchiano2021stochastic,wachi2018safe} and Bayesian optimization~\citep{sui2015safe}. However, the type of exploration of these works is the classical strategic exploration. Our motivation to explore is somewhat different, and we aim at improving the overall user experience in recommender systems by presenting previously unseen actions and discovering potentially better options for users. Thus, we pursue safety and user exploration (not strategic exploration) in the OPL setup and propose a practically implementable and safe method to balance the conflicting objectives of value 
and novelty.

\paragraph{\textbf{Off-Policy Learning with Novel Actions}}
OPL from logged data is a prominent and practical approach for policy learning, as it does not require any risky and costly online interactions~\citep{levine2020offline,saito2021counterfactual}. 
Conventional methods can be categorized into value-based~\citep{jeunen2021pessimistic, kumar2020conservative,le2019batch} and gradient-based~\citep{chen2019top, joachims2018deep, sachdeva2020off,swaminathan2015batch, swaminathan2015self} approaches. Value-based OPL trains a reward model based on the logged data and uses the predicted rewards to make downstream decisions.
Although this approach works reasonably well if the reward prediction is accurate, offline reward prediction often fails to extrapolate due to distributional shift in the logged data~\citep{saito2021evaluating}.
In particular, overestimation of the value of unseen actions is problematic, as it may lead to a poor policy~\citep{fujimoto2019off}.
To avoid the overestimation issue, recent literature focuses more on conservative value estimates~\citep{jeunen2021pessimistic, kumar2019stabilizing, kumar2020conservative}. However, this approach can also be unsafe in the presence of novel actions as demonstrated in our experiments.
Gradient-based approach is a different paradigm based on an unbiased policy gradient estimated via importance sampling~\citep{precup2000eligibility, strehl2010learning}. The benefit of using an unbiased policy gradient is that it steadily improves the policy value with growing sample size. However, an unbiased policy gradient heavily depends on the so-called \textit{full support} assumption, which requires the logging policy to allocate non-zero probabilities to all context-action pairs, which is arguably broken when novel actions exist. This makes the learned policy too conservative in our setup, and as a result, it fails to explore and identify high-quality novel actions~\citep{london2020offline, sachdeva2020off}.

One of the related work,~\citet{sachdeva2020off}, study the OPL problem under \textit{deficient support}. Deficient support is a relaxed case of our setting where there is no novel action, but some context-action pairs have zero probability under the logging policy.~\citet{sachdeva2020off} compare several techniques including the one penalizing the policy that is likely to choose deficient actions. This strategy of being conservative about deficient actions is promising to maximize the policy value. However, this method prefers a policy similar to the logging policy, suggesting the difficulty of ensuring safety and exploring novel actions at the same time. To discover high-quality deficient actions, \citet{tran2021combining} propose to explore uncertain actions in the online environment until the policy becomes confident enough about the value of deficient actions.
While this method works with less number of online interactions compared to fully online learning, exploration in the initial phase can still be unsafe. 
Moreover, frequent update of a policy incurs large implementation costs, making this type of \textit{offline-to-online} policy learning methods, including~\citep{li2021unifying, xu2021safely}, often infeasible.

Compared to these related methods, our approach works with much more reasonable implementation efforts and still succeeds in pursuing novelty while guaranteeing safety. We achieve this by allowing only a few additional policy deployments to leverage the safety margin to strengthen user exploration when possible. The idea of allowing only several deployments is closely related to the concept of \textit{deployment efficient} policy learning first formulated by~\citet{matsushima2021deployment} and later analyzed theoretically by~\citet{huang2022towards}. This framework allows deploying a few number of policies to outperform fully offline policy in terms of policy value in a tractable and easily implementable way. We take a step further and tackle a more challenging problem of exploring completely novel actions while guaranteeing safety at the same time, which is made possible with our Safe OPG and carefully designed safety constraints in DEPSUE.

\paragraph{\textbf{Safe Policy Learning}}
In many applications of bandits and RL, including recommender systems, robotics, autonomous driving, and precision medicine, \textit{guaranteeing} safety is often critical to prevent catastrophic failures (such as car accidents)~\citep{dulac2020empirical, garcia2015comprehensive,xu2021constraints,wachi2018safe}. This motivates various algorithms for solving a constrained policy learning problem~\citep{achiam2017constrained, laroche2019safe, xu2021crpo, yu2019convergent}. For instance, the primal-dual approach casts the policy optimization problem into a minmax optimization via Lagrangian relaxation~\citep{bertsekas1997nonlinear, chow2017risk, liang2018accelerated}.
\citet{paternain2019constrained} justified the effectiveness of primal-dual methods, in
which zero duality gap is guaranteed. A recent work also established the convergence rate of the primal-dual method under Slater’s condition assumption~\citep{ding2021provably}. \citet{yu2019convergent}
proposed a constrained policy gradient algorithm with convergence guarantee by solving a sequence of surrogate convex constrained optimization problems. \citet{xu2021crpo} propose an algorithm, which is implementable as easy as unconstrained policy optimization and has a global optimality guarantee as well.
In the offline setting, a few safe RL methods exist, solving the primal-dual optimization by estimating the safety constraint with model-based predictors~\citep{le2019batch, xu2021constraints}. In particular, \citet{xu2021constraints} develop Constraints Penalized Q-Learning (CPQ), the first continuous control RL algorithm capable of learning from mixed offline data under a safety constraint. However, these previous methods for safe offline RL rely on a safety constraint accurately approximated with only offline data. This should be avoided in particular when novel actions exist. This is because, if we fail to accurately approximate the constraint, safety may no longer be ensured in the test time. In contrast, our Safe OPG is developed specifically for the seting with novel actions. To our knowledge, ours is the first attempt to enforce safety in a way to be applicable to the setting with novel actions. More specifically, Safe OPG builds on top of HCOPE~\citep{thomas2015confidence, thomas2015high} and adopts an outer-loop validation scheme during training so that the policy is independent of the data used to test the safety condition. This validation scheme ensures that HCOPE lower bounds the policy value with high probability without relying on an accurate model-based approximation of the safety constraint as done in previous work.

\section{Convergence analysis} \label{app:proof}
Here, we theoretically analyze the convergence of Safe OPG.
To prepare, we first assume three assumptions.

The first assumption assumes that the objective function is bounded.
\begin{assumption} (Boundedness of the objective function) \label{assm:bound}
    The objective function is bounded, s.t., $\hat{V}(\pi; \calD_0^{\mathrm{(S1)}}) < \infty$.
\end{assumption}

Then, we assume that there exists at least one policy that satisfies the safety constraint in the predefined policy class.
\begin{assumption} (Realizability of the safety constraint) \label{assm:constraint}
    The safety constraint is realizable, s.t., $\exists \pi \in \Pi, \mathrm{s.t.,} \hat{V}_{-}(\pi; \calD_0^{\mathrm{(S2)}}) > C$.
\end{assumption}

Finally, we also assume that the maximization of the regularization function leads to the constraint satisfaction.
\begin{assumption} (Generalization of the regularization function) \label{assm:regularization}
    The regularization function $\mathcal{R}(\cdot)$ is generalizable, s.t., $\exists C'$,
    \begin{align*}
        \mathcal{R}(\pi; \calD_0^{\mathrm{(S1)}}) > C' \Rightarrow \hat{V}_{-}(\pi; \calD_0^{\mathrm{(S2)}}) > C.
    \end{align*}
\end{assumption}

Under Assumption~\ref{assm:bound}-\ref{assm:regularization}, we show that the convergence point (i.e., local saddle point) of Safe OPG exists.
\begin{theorem} (Existence of a local saddle point) \label{thrm:saddle_point}
    Under the Assumptions~\ref{assm:bound}-\ref{assm:regularization}, there exists a local saddle point $(\pi^{\star}, \lambda^{\star})$ that satisfies, given $\calD_0^{\mathrm{(S1)}}$, $\forall \epsilon > 0, \pi \in \Pi, \lambda \in \mathbb{R}_{+}$,
    \begin{align*}
        \mathcal{L}(\pi; \lambda^{\star}, \calD_0^{\mathrm{(S1)}}) \leq \mathcal{L}(\pi^{\star}; \lambda^{\star}, \calD_0^{\mathrm{(S1)}}) \leq \mathcal{L}(\pi; \lambda, \calD_0^{\mathrm{(S1)}}).
    \end{align*}
\end{theorem}

\begin{proof}
Assumption~\ref{assm:regularization} ensures that the solution of the following constrained optimization problem can also be a solution to that of Eq.~\eqref{eq:constrained_lower_bound}.
\begin{align}
    \begin{array}{cl}
        \displaystyle \max_{\pi \in \Pi} & \hat{V}(\pi; \calD_0^{\mathrm{(S1)}}) \\
        \mathrm{s.t.} & \mathcal{R}(\pi; \calD_0^{\mathrm{(S1)}}) > C'
    \end{array} \label{eq:regularization_constrained}
\end{align}
Then, using the Lagrangian relaxation technique, Eq~\eqref{eq:regularization_constrained} is transformed to the following minmax optimization problem.
\begin{align}
    \min_{\lambda \in \mathbb{R}_{+}} \max_{\pi \in \Pi} \, \hat{V}(\pi; \calD_0^{\mathrm{(S1)}}) + \lambda \left( \mathcal{R}(\pi; \calD_0^{\mathrm{(S1)}}) - C' \right)
    \label{eq:dual_regularization}
\end{align}
\citet{bertsekas1997nonlinear} shows that a local saddle point ($\pi^{\star}$, $\lambda^{\star}$) exists for Eq.~\eqref{eq:dual_regularization} and that $\pi^{\star}$ satisfy the given safety constraint in Eq.~\eqref{eq:regularization_constrained} under Assumption~\ref{assm:bound} and ~\ref{assm:constraint}.
\end{proof}

\begin{table*}[htb]
\large
\centering
\caption{Policy value and relative optimality of $\pi_0$} \label{tab:behavior}
\def\arraystretch{1.2}
\scalebox{0.8}{
\begin{tabular}{c|cccccccc}
\toprule
\textbf{logging policies} & $\beta=-24$ & 
$\beta=-16$ & 
$\beta=-8$ & 
$\beta=0$ &
$\beta=8$ & 
$\beta=16$ & 
$\beta=24$ &
\\\midrule \midrule
$V(\pi_0)$ & 0.200 & 0.213 & 0.234 & 0.253 & 0.268 & 0.280 & 0.289 \\
$V(\pi_0) / V(\pi^{\ast})$ & 0.647 & 0.690 & 0.757 & 0.820 & 0.869 & 0.908 & 0.938 \\
\bottomrule
\end{tabular}
}
\end{table*}

\section{Details and additional results of the semi-synthetic experiment} \label{app:experiment}
Here, we describe some additional experimental setups omitted in the main text and report some additional experimental results.

\subsection{Additional Settings}
Our implementation uses \textit{PyTorch}~\citep{paszke2019pytorch}.

\paragraph{\textbf{Simulator}} Following \citet{ma2020off}, we define simulator based on embedding modules and a reward prediction module. First, we transform each categorical feature of the user demographics and item features to 5-dimensional embeddings and concatenate them to 10-dimensional embeddings which are unique for each user and movie. Consequently, we obtain 30-dimensional context ($x$) and 20-dimensional action context ($e_a$), respectively. We also define $q(x, e_a)$ with a 3-layer MLP with ReLU activation. The latent dimensions are 100 and 50, respectively. 
We train both embedding modules and reward prediction module simultaneously to minimize the binary cross entropy loss in predicting binarized ratings of the movielens dataset.

\paragraph{\textbf{Logging Policy}}
Table~\ref{tab:behavior} shows the performance and the relative optimality of the logging policy compared to the optimal one ($V(\pi_0) / V(\pi^{\ast})$) averaged over 30 random seeds.

\paragraph{\textbf{Reward Model}}
All reward models $\hat{q}$ (both naive and CQL) use the same 3-layer MLP with ReLU activation, whose latent dimensions are 100 and 10. The naive reward model estimates the mean reward function via emsembling as $\sum_{j=1}^5 \hat{q}_j / 5$ and $\min_{j \in [5]} \hat{q}_j$, where each $\hat{q}_j$ is trained on the bootstrapped logged data to minimize the binary cross entropy loss. In contrast, CQL aims to mitigate the overestimation of uncertain actions as follows.
\begin{align*}
    \hat{q} & \leftarrow \min_{q} \mE_n [ (r_i - \hat{q}(x_i, a_i))^2 ] \\
    & \quad + \alpha \mE_n [ \max_{\pi'} \hat{q}(x_i, \pi'(x_i)) - \hat{q}(x_i, a_i)) ],
\end{align*}
where the first term minimizes the estimation error. To learn $\hat{q}$ in a manner similar to the naive one, we replace the MSE loss (the first term) to the binary cross entropy loss in our experiment. In contrast, the CQL-specific loss function is the second term, which penalizes the overestimation of the values of uncertain actions. We set $\alpha=2.0$ to balance the two objectives. Although the second term of CQL originally minimizes the overestimation in the worst case in terms of policy (i.e., $\max_{\pi'}$), this term can be computationally costly to calculate the gradient when $|\calA|$ is large. Therefore, we simply substitute $\max_{\pi'}$ to the uniform random policy and randomly sample negative actions ($\pi'(x_i)$) in our experiment. We show the estimation result of each reward model in Figure~\ref{fig:true_reward}-\ref{fig:cql_reward}. We also discuss their relation to the OPL results in Appendix~\ref{app:results}.

\paragraph{\textbf{Policy Network and Parameters}}
Policy network $f_{\psi}$ consists of a 2-layer MLP with ReLU activation, which maps $d_x$-dimensional input (i.e., $x$) to 100 dimension latent vector, and then outputs $|A|$-dimensional values (i.e., $f_{\psi}(x, a)$) used for softmax. We use SGD~\citep{bottou2010large} as its optimizer, with the learning rate $\eta_{\psi}=0.1$ (naive OPGs) and $\eta_{\psi}=0.001$ (Safe OPG), based on the convergence analysis.
We also set $S=10,000$ (gradient steps) and $\eta_{\lambda}=0.01$ (step size to adjust $\lambda$). Note that, to further alleviate the vanishment of the gradient, we multiple $1 / \max_{(x, a) \in \calX, \calA} \pi(a|x)$ to the gradient of all compared methods.

\paragraph{\textbf{Validation with OPE estimators}}
To see if OPE can validate the safety of the policies learned by OPG, we additionally conduct an OPE experiment. In this experiment, we use 50,000 samples of the validation dataset collected by $\pi_0$, which is independent from the training dataset. In the hypothesis test, we set $H_0: V(\pi) = V(\pi_0)$ and $H_1: V(\pi) > V(\pi_0)$ as the null and alternative hypothesis, respectively. Using the validation OPE estimate $\hat{V}(\cdot)$, we regard the sample with $\hat{V}(\pi; \calD_0) > V_{\mathrm{on}}(\pi_0; \calD_0)$ as positive, and otherwise regard the one as negative. 
We calculate the Type I and Type II error rates of each OPE estimator over the policies learned from various logging policies (i.e., total 210 samples of $\pi$ learned in the OPL experiment, which consists of seven different values of $\beta$ and 30 random seeds for each). 

We use Direct Method (DM), Inverse Propensity Scoring (IPS), and Doubly Robust (DR) to evaluate the learned policies. DM is a model-based OPE estimator, which estimates the policy value as:
\begin{align*}
    \hat{V}_{\mathrm{DM}}(\pi; \calD_0) := \mE_n \left[ \mE_{a \sim \pi(a|x_i)} \left[ \hat{q}(x_i, e_a) \right] \right].
\end{align*}
While DM is consistent when $\hat{q}$ is accurate, it is vulnerable to the extrapolation error. DM can be inaccurate especially for novel actions, as its value $\hat{q}$ entails huge uncertainty.

In contrast, IPS estimates the policy value using the importance sampling technique as follows.
\begin{align*}
    \hat{V}_{\mathrm{IPS}}(\pi; \calD_0) := \mE_n [ w(x_i, a_i) r_i ],
\end{align*}
where $w(x_i, a_i) = \pi(a_i | x_i) / \pi_0(a_i | x_i)$ is the importance weight. IPS enables unbiased estimation when $\pi_0$ is full support (i.e., $\forall (x, a) \in \calX \times \calA, \pi(x, a) > 0 \rightarrow \pi_0(x, a) > 0$). However, this support condition is not satisfied in the presence of novel actions, making IPS infeasible to evaluate the values of novel actions. Moreover, IPS is also known to suffer from high variance when $\pi$ deviate from $\pi_0$ greatly.

Finally, DR is a hybrid of DM and IPS, which estimates the policy value as follows.
\begin{align*}
    & \hat{V}_{\mathrm{DR}}(\pi; \calD_0) \\
    &:= \mE_n \left[ \mE_{a \sim \pi(a|x_i)} \left[ \hat{q}(x_i, e_a) \right] + w(x_i, a_i) (r_i - \hat{q}(x_i, e_{a_i})) \right].
\end{align*}
DR uses $\hat{q}$ as a baseline estimation as well as DM, but also corrects the estimation bias of $\hat{q}$ in a data-driven manner using importance sampling. 
Note that, due to the deficiency in $\calD_0$, DR relies solely on DM for the novel actions.

\begin{figure}[htb]
  \begin{minipage}[b]{0.49\linewidth}
    \centering
    \includegraphics[width=1.0\linewidth]{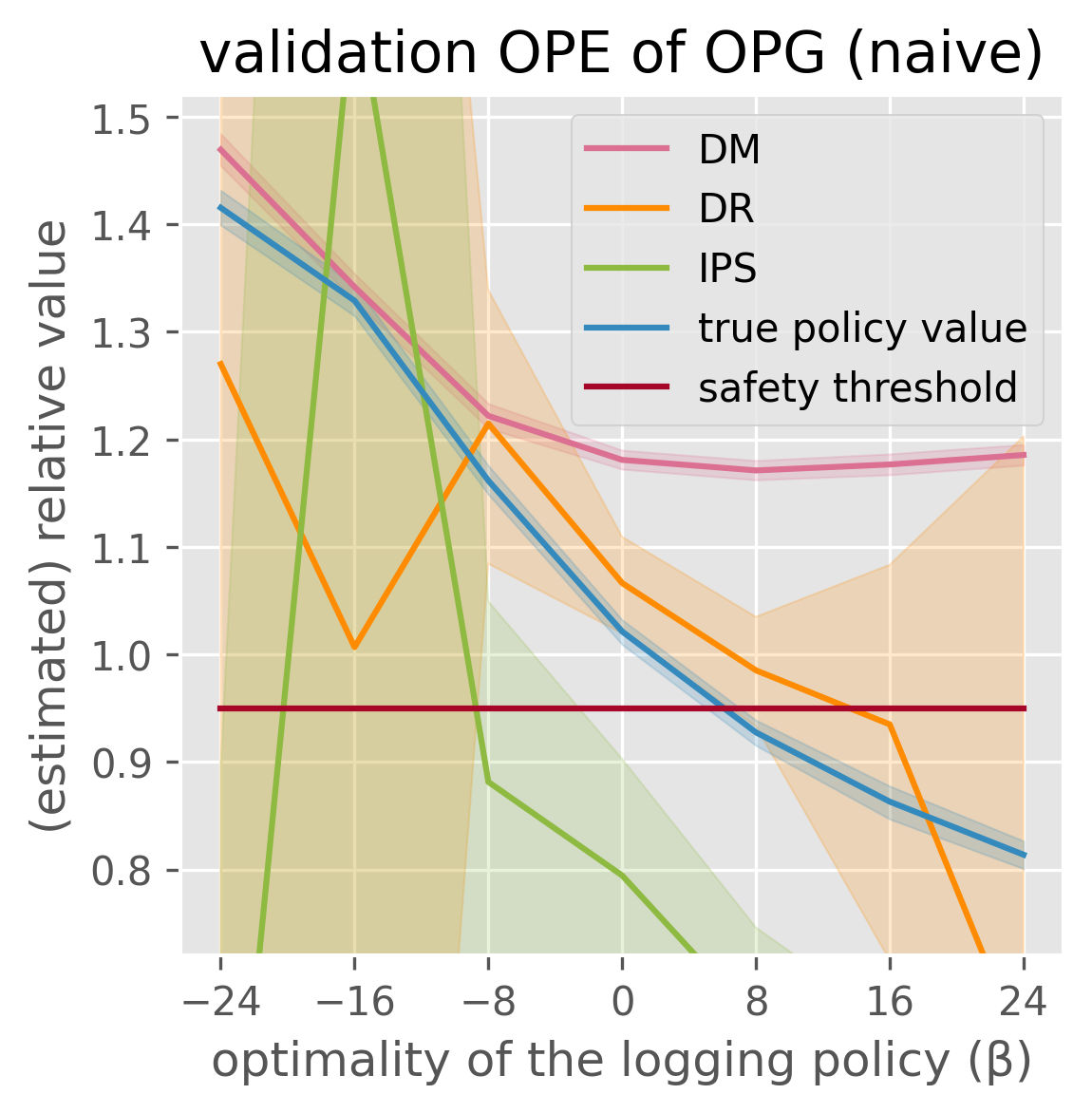}
  \end{minipage}
  \begin{minipage}[b]{0.49\linewidth}
    \centering
    \includegraphics[width=1.0\linewidth]{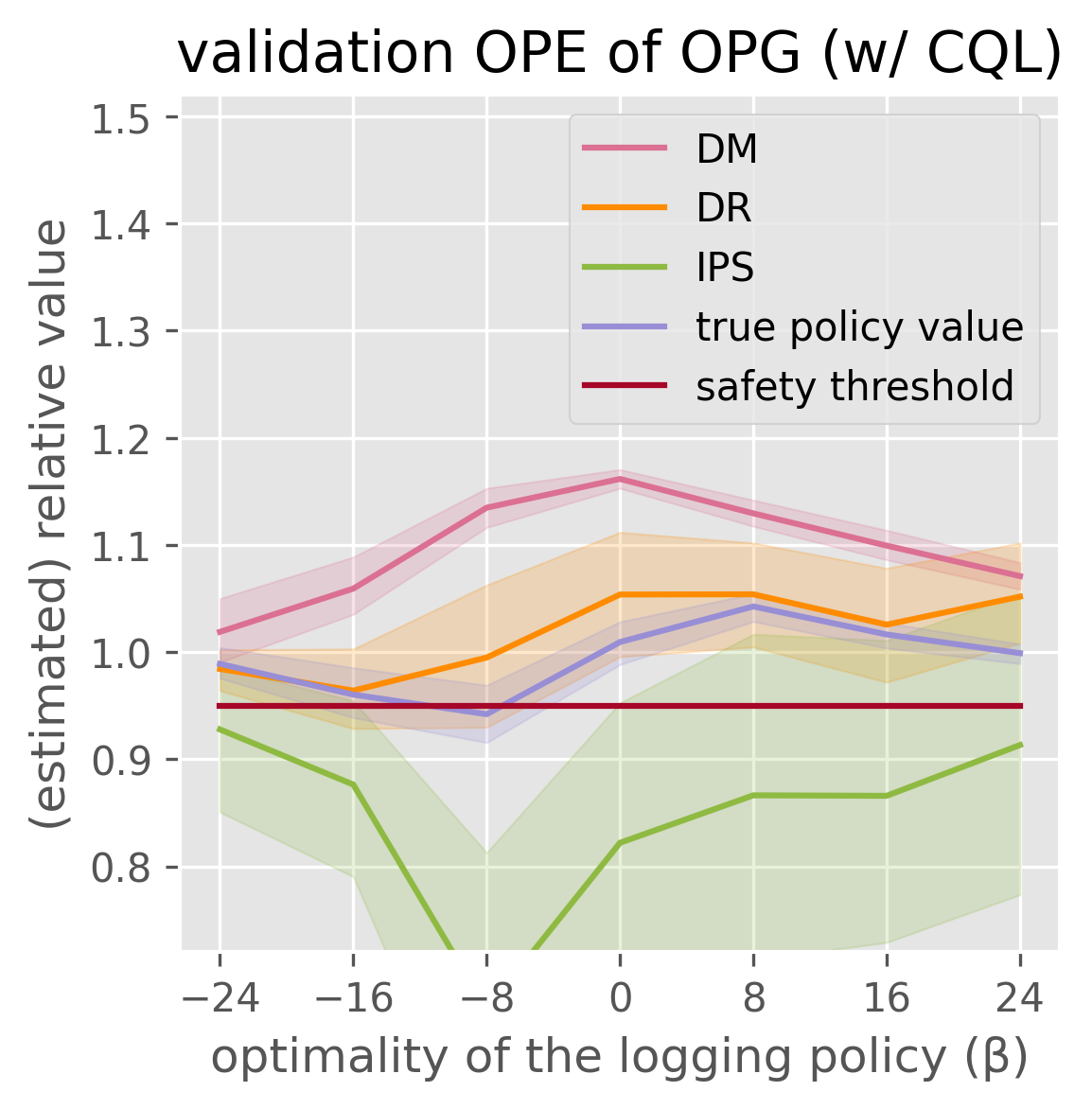}
  \end{minipage}
  \caption{Relative value of the policies of (Left) OPG (naive) and (Right) OPG (w/ CQL) estimated by various validation OPE estimators}
  \label{fig:validation}
\end{figure}

\begin{table*}[htb]
\large
\centering
\caption{Type I and type II error rates of the hypothesis test implemented with validation OPE} \label{tab:validation}
\scalebox{0.8}{
\begin{tabular}{ccccccc}
    \toprule
    OPL method && \multicolumn{2}{c}{OPG (naive)} && \multicolumn{2}{c}{OPG (w/ CQL)}
    \\ \cmidrule{1-1} \cmidrule{3-4} \cmidrule{6-7}
    OPE estimator &&
    Type I Error Rate & Type II Error Rate && Type I Error Rate & Type II Error Rate \\ \midrule \midrule
    IPS && 0.780 & 0.096 && 0.417 & 0.319 \\
    DM && 0.000 & 1.000 && 0.037 & 1.000 \\
    DR && 0.189 & 0.723 && 0.209 & 0.468 \\
    \bottomrule
\end{tabular}
}
\vskip 0.1in
\raggedright
\fontsize{8.0pt}{8.0pt}\selectfont \textit{Note}: A lower value is better for the both metrics. Our null and alternative hypothesis are $H_0: V(\pi) = V(\pi_0)$ and $H_1: V(\pi) > V(\pi_0)$, respectively. 
\end{table*}

\subsection{Results of Validation OPE} \label{app:validation_ope}
We conduct validation OPE on the policies learned by OPG (naive) and OPG (w/ CQL) and show the results in Figure~\ref{fig:validation}. We also report Type I and Type II error rates of the hypothesis test based on validation OPE in Table~\ref{tab:validation}. 
First, we observe that DM and DR tend to be too optimistic in deploying unsafe policies, which is unacceptable in practice. Specifically, DM records 1.000 in Type II error rate in both cases, indicating that DM cannot detect any unsafe policies due to overestimation of the policy value. DR slightly alleviate the overestimation, however, its Type I error rates are still high and almost half of the unsafe policies are mistakenly allowed for the deployment.

Whereas, we observe that IPS suffers from impractically variance, resulting in different OPE results between OPG (naive) and OPG (w/ CQL). 
In particular, we observe that IPS records high Type II error rate on OPG (w/ CQL), allowing more than 30\% of unsafe policies to be deployed.
In contrast, we observe that IPS produces high Type I error rate on OPG (naive), rejecting almost all policies regardless of the true safety (this observation is also attributed to the fact that IPS cannot evaluate novel actions due to deficient support~\citep{sachdeva2020off, london2020offline}). These results suggest that OPE can be unreliable when there exist novel actions or a large discrepancy between the logging and evaluation policies. Therefore, we need an OPL method that satisfies the safety constraint from the policy training phase.

\subsection{Additional Comparisons of OPL methods} \label{app:results}
We provide the comprehensive statistics of the policies learned by the compared OPL methods in Table~\ref{tab:violation}-\ref{tab:novelty}. 
First, Table~\ref{tab:violation} shows the proportion of the learned policies violating the safety constraint. We observe that OPG (naive) violate the constraint in most cases when $\pi_0$ performs well. In contrast, OPG (w/ CQL) violate the constraint when $\pi_0$ performs poorly. We attribute this observation to the quality of $\hat{q}$. To verify this, we compare the predicted reward of each $\hat{q}$ on both observed actions and novel actions in Figure~\ref{fig:true_reward}-\ref{fig:cql_reward}. 
We observe that the naive reward model estimates the values of novel actions as high as observed actions, whereas, in reality, novel actions are poor compared to the observed actions when $\beta=16$. 
This leads to the optimistic behaviors of OPG (naive) in terms of novel actions, which eventually deteriorate the value in near-optimal logging policy cases ($\beta > 0$).
On the other hand, CQL underestimates novel actions even when $\pi_0$ performs poorly ($\beta < 0$), which means that CQL is overconfident about observed actions. As a result, CQL in turn aggravates the performance of $\pi_0$ when $\pi_0$ is poor by relying too much on the (sub-optimal) observed actions. These results suggest the difficulty of dealing with safety issues in a model-based approach.

In contrast, Table~\ref{tab:violation} shows that our DEPSUE satisfies the safety constraint in almost every experimental conditions (i.e., a range of logging policies and $K$) except for a few cases of violation in $\beta=16$ for $K=5$. These safe behaviors of DEPSUE are due to our model-free approach in dealing with the uncertainty of novel actions. Moreover, even in a few cases of the constraint violation, DEPSUE ($K=5$) gains relatively large value in its worst case compared to the baselines. Specifically, we observe that the worst case relative value of DEPSUE ($K=5$) is 0.936, while those of OPG (naive) and OPG (w/ CQL) are 0.684 and 0.819, respectively. Thus, we argue that DEPSUE guarantees much safer behavior compared to the baseline OPG methods, while sucessfully exploring novel actions.

Next, we show the proportion of the learned policy acquiring novelty more than 0.1 and the mean and standard deviation of the novelty in Table~\ref{tab:novel_rate} and \ref{tab:novelty}, respectively. Comparing Table~\ref{tab:novelty}, we observe that the mean novelty of the baseline OPG methods is higher than that of DEPSUE. However, since the novelty of the baseline OPG heavily depends on the individual quality of $\hat{q}$, we observe that standard deviation of the novelty is also high for the baseline OPGs. Thus, the baseline methods only acquires novelty in limited cases, as shown in Table~\ref{tab:novel_rate}. On contrary to these observations, our DEPSUE gains smaller novelty on average as it becomes conservative about novel actions. However, DEPSUE gains adequate novelty ($\geq 0.1$) in more cases compared to the baseline method, enabling safely exploration of novel actions more constantly.

Finally, Figure~\ref{fig:5stage} shows how DEPSUE ($K=5$) explores novel actions stage by stage while ensuring safety. 
The result demonstrates that both relative value and novelty improve as the safety regularization is gradually relaxed. 
In particular, we observe that DEPSUE is relatively aggressive when the logging policies are poor and stochastic (i.e., $\beta=-16,-8,0$) This is because the HCOPE's lower bound more easily improves in the stochastic logging policy cases than others and the safety margin becomes larger in poor logging policy cases. On the other hand, we observe that DEPSUE is conservative when the logging policy is already near-optimal, as the safety margin is small. In this way, DEPSUE enables safe user exploration in a range of logging policies, even in the presence of novel actions.

\begin{table*}[h]
\begin{tabular}{cc}
\begin{minipage}{1.0\textwidth}
\large
\centering
\caption{Proportion of the learned policies violating the safety constraint in the semi-synthetic experiment}
\def\arraystretch{1.2}
\scalebox{0.80}{
\begin{tabular}{cccccccc}
\toprule
\textbf{OPL methods} & $\beta=-24$ & $\beta=-16$ & $\beta=-8$ & $\beta=0$ & $\beta=8$ & $\beta=16$ & $\beta=24$  \\\midrule \midrule
\textbf{OPG} (naive) 
& 0/30 & 0/30 & 0/30
& 2/30 & \textcolor{blue}{$\mathbf{21/30^{\dagger}}$} & \textcolor{blue}{$\mathbf{30/30^{\dagger}}$} 
& \textcolor{blue}{$\mathbf{30/30^{\dagger}}$} \\
\textbf{OPG} (w/ CQL) 
& 1/30 & \textcolor{blue}{$\mathbf{17/30^{\dagger}}$} & \textcolor{blue}{$\mathbf{19/30^{\dagger}}$}
& \textcolor{blue}{$\mathbf{6/30^{\dagger}}$} & 0/30 & 1/30 
& \textcolor{blue}{$\mathbf{3/30{\dagger}}$} \\ \midrule
\textbf{DEPSUE} ($K=1$) 
& 0/30 & 0/30 & 0/30
& 0/30 & 0/30 & 0/30 
& 0/30 \\
\textbf{DEPSUE} ($K=2$) 
& 0/30 & 0/30 & 0/30
& 0/30 & 0/30 & 0/30 
& 0/30 \\ 
\textbf{DEPSUE} ($K=5$) 
& 0/30 & 0/30 & 0/30
& 0/30 & 1/30 & \textcolor{blue}{$\mathbf{3/30^{\dagger}}$}
& 1/30 \\
\bottomrule
\end{tabular}
}
\vskip 0.1in
\raggedright
\fontsize{8.5pt}{8.5pt}\selectfont \textit{Note}:
A lower value is better. The \textcolor{blue}{$\mathbf{blue^{\dagger}}$} fonts represent the OPL methods that violate the safety constraints in more than two cases.
\label{tab:violation}
\end{minipage}
\end{tabular}
\end{table*}

\begin{table*}[h]
\begin{tabular}{cc}
\begin{minipage}{1.0\textwidth}
\large
\centering
\caption{Mean and worst case relative policy value in the semi-synthetic experiment}
\def\arraystretch{1.2}
\scalebox{0.65}{
\begin{tabular}{cccccccc}
\toprule
\textbf{OPL methods} & $\beta=-24$ & $\beta=-16$ & $\beta=-8$ & $\beta=0$ & $\beta=8$ & $\beta=16$ & $\beta=24$  \\\midrule \midrule
\textbf{OPG} (naive) 
& 1.415 ($\geq$\textcolor{dkred}{$\mathbf{1.320^{\ast}}$}) & \textcolor{dkred}{$\mathbf{1.329^{\ast}}$} ($\geq$\textcolor{dkred}{$\mathbf{1.226^{\ast}}$}) & 1.162 ($\geq$\textcolor{dkred}{$\mathbf{1.108^{\ast}}$})
& \textcolor{dkred}{$\mathbf{1.021^{\ast}}$} ($\geq$\textcolor{blue}{$\mathbf{0.924^{\dagger}}$}) & \textcolor{blue}{$\mathbf{0.927^{\dagger}}$} ($\geq$\textcolor{blue}{$\mathbf{0.860^{\dagger}}$}) & \textcolor{blue}{$\mathbf{0.863^{\dagger}}$} ($\geq$\textcolor{blue}{$\mathbf{0.708^{\dagger}}$}) 
& \textcolor{blue}{$\mathbf{0.813^{\dagger}}$} ($\geq$\textcolor{blue}{$\mathbf{0.684^{\dagger}}$}) \\
\textbf{OPG} (w/ CQL) 
& 0.989 ($\geq$\textcolor{blue}{$\mathbf{0.937^{\dagger}}$}) & 0.960 ($\geq$\textcolor{blue}{$\mathbf{0.889^{\dagger}}$}) & \textcolor{blue}{$\mathbf{0.942^{\dagger}}$} ($\geq$\textcolor{blue}{$\mathbf{0.819^{\dagger}}$})
& 1.009 ($\geq$\textcolor{blue}{$\mathbf{0.901^{\dagger}}$}) & \textcolor{dkred}{$\mathbf{1.043^{\ast}}$} ($\geq$0.972) & \textcolor{dkred}{$\mathbf{1.016^{\ast}}$} ($\geq$\textcolor{blue}{$\mathbf{0.931^{\dagger}}$})
& \textcolor{dkred}{$\mathbf{0.999^{\ast}}$} ($\geq$\textcolor{blue}{$\mathbf{0.935^{\dagger}}$}) \\ \midrule
\textbf{DEPSUE} ($K=1$) 
& 1.037 ($\geq$1.026) & 1.063 ($\geq$1.029) & 1.074 ($\geq$1.060)
& 1.011 ($\geq$\textcolor{dkred}{$\mathbf{1.002^{\ast}}$}) & 0.982 ($\geq$\textcolor{dkred}{$\mathbf{0.974^{\ast}}$}) & 0.981 ($\geq$\textcolor{dkred}{$\mathbf{0.979^{\ast}}$})
& 0.982 ($\geq$\textcolor{dkred}{$\mathbf{0.980^{\ast}}$}) \\
\textbf{DEPSUE} ($K=2$) 
& 1.141 ($\geq$1.022) & 1.145 ($\geq$1.127) & 1.102  ($\geq$1.074)
& 1.011 ($\geq$0.998) & 0.969 ($\geq$0.959) & 0.972 ($\geq$0.954)
& 0.982 ($\geq$0.960) \\ 
\textbf{DEPSUE} ($K=5$) 
& 1.254 ($\geq$1.174) & 1.201 ($\geq$1.089) & 1.104 ($\geq$1.079)
& 1.015 ($\geq$0.991) & 0.972 ($\geq$\textcolor{blue}{$\mathbf{0.944^{\dagger}}$}) & 0.978 ($\geq$\textcolor{blue}{$\mathbf{0.936^{\dagger}}$})
& 0.982 ($\geq$\textcolor{blue}{$\mathbf{0.950^{\dagger}}$}) \\
\bottomrule
\end{tabular}
}
\vskip 0.1in
\raggedright
\fontsize{8.5pt}{8.5pt}\selectfont \textit{Note}:
A higher value is better. 
The \textcolor{dkred}{$\mathbf{red^{\ast}}$} fonts represent the best OPL methods. The \textcolor{blue}{$\mathbf{blue^{\dagger}}$} fonts represent the OPL methods that violate the safety constraint (i.e., 0.950).
\label{tab:value}
\end{minipage}
\end{tabular}
\end{table*}

\begin{table*}[h]
\begin{tabular}{cc}
\begin{minipage}{1.0\textwidth}
\large
\centering
\caption{Proportion of the learned policies acquiring novelty more than 0.1 in the semi-synthetic experiment}
\def\arraystretch{1.2}
\scalebox{0.80}{
\begin{tabular}{cccccccc}
\toprule
\textbf{OPL methods} & $\beta=-24$ & $\beta=-16$ & $\beta=-8$ & $\beta=0$ & $\beta=8$ & $\beta=16$ & $\beta=24$  \\\midrule \midrule
\textbf{OPG} (naive) 
& 4/30 & 6/30 & 14/30
& 18/30 & \textcolor{blue}{$\mathbf{21/30^{\dagger}}$} & \textcolor{blue}{$\mathbf{17/30^{\dagger}}$} 
& \textcolor{blue}{$\mathbf{11/30^{\dagger}}$} \\
\textbf{OPG} (w/ CQL) 
& 1/30 & \textcolor{blue}{$\mathbf{2/30^{\dagger}}$} & \textcolor{blue}{$\mathbf{9/30^{\dagger}}$}
& \textcolor{blue}{$\mathbf{5/30^{\dagger}}$} & 5/30 & \textcolor{dkred}{$\mathbf{5/30^{\ast}}$} 
& \textcolor{blue}{$\mathbf{4/30^{\dagger}}$} \\ \midrule
\textbf{DEPSUE} ($K=1$) 
& 0/30 & 1/30 & \textcolor{dkred}{$\mathbf{28/30^{\ast}}$}
& 1/30 & 0/30 & 0/30 
& 0/30 \\
\textbf{DEPSUE} ($K=2$) 
& 4/30 & 8/30 & 21/30
& \textcolor{dkred}{$\mathbf{30/30^{\ast}}$} & 13/30 & 0/30 
& 0/30 \\ 
\textbf{DEPSUE} ($K=5$) 
& \textcolor{dkred}{$\mathbf{15/30^{\ast}}$} & \textcolor{dkred}{$\mathbf{15/30^{\ast}}$} & 27/30
& 29/30 & \textcolor{dkred}{$\mathbf{16/30^{\ast}}$} & \textcolor{blue}{$\mathbf{1/30^{\dagger}}$} 
& 0/30 \\
\bottomrule
\end{tabular}
}
\vskip 0.1in
\raggedright
\fontsize{8.5pt}{8.5pt}\selectfont \textit{Note}:
A higher value is better. 
The \textcolor{dkred}{$\mathbf{red^{\ast}}$} fonts represent the OPL methods that acquire adequate novelty ($N(\pi) \geq 0.1$) in most cases among those satisfying the safety constraint. The \textcolor{blue}{$\mathbf{blue^{\dagger}}$} fonts represent the OPL methods that violate the safety constraint in more than two cases.
\label{tab:novel_rate}
\end{minipage}
\end{tabular}
\end{table*}

\begin{table*}[h]
\begin{tabular}{cc}
\begin{minipage}{1.0\textwidth}
\large
\centering
\caption{Mean and standard deviation of the policy novelty in the semi-synthetic experiment}
\def\arraystretch{1.2}
\scalebox{0.65}{
\begin{tabular}{cccccccc}
\toprule
\textbf{OPL methods} & $\beta=-24$ & $\beta=-16$ & $\beta=-8$ & $\beta=0$ & $\beta=8$ & $\beta=16$ & $\beta=24$  \\\midrule \midrule
\textbf{OPG} (naive) 
& 0.133 ($\pm$0.313) & 0.196 ($\pm$0.355) & 0.212 ($\pm$0.282)
& 0.237 ($\pm$0.263) & 0.265 ($\pm$0.268) & 0.202 ($\pm$0.233) 
& 0.146 ($\pm$0.209) \\
\textbf{OPG} (w/ CQL) 
& 0.034 ($\pm$0.179) & 0.069 ($\pm$0.247) & 0.299 ($\pm$0.443)
& 0.168 ($\pm$0.336) & 0.169 ($\pm$0.359) & 0.164 ($\pm$0.361) 
& 0.132 ($\pm$0.335) \\ \midrule
\textbf{DEPSUE} ($K=1$) 
& 0.018 ($\pm$0.003) & 0.052 ($\pm$0.023) & 0.120 ($\pm$0.013)
& 0.086 ($\pm$0.006) & 0.045 ($\pm$0.004) & 0.019 ($\pm$0.002) 
& 0.010 ($\pm$0.001) \\
\textbf{DEPSUE} ($K=2$) 
& 0.075 ($\pm$0.027) & 0.112 ($\pm$0.048) & 0.124 ($\pm$0.048)
& 0.134 ($\pm$0.021) & 0.098 ($\pm$0.016) & 0.037 ($\pm$0.015) 
& 0.021 ($\pm$0.008) \\
\textbf{DEPSUE} ($K=5$) 
& 0.139 ($\pm$0.102) & 0.137 ($\pm$0.080) & 0.169 ($\pm$0.058)
& 0.157 ($\pm$0.047) & 0.097 ($\pm$0.049) & 0.032 ($\pm$0.035) 
& 0.012 ($\pm$0.017) \\
\bottomrule
\end{tabular}
}
\vskip 0.1in
\raggedright
\fontsize{8.5pt}{8.5pt}\selectfont \textit{Note}:
A higher value is better.
\label{tab:novelty}
\end{minipage}
\end{tabular}
\end{table*}

\begin{figure*}[htbp]
  \begin{minipage}[b]{0.98\linewidth}
  \centering
      \begin{minipage}[b]{0.33\linewidth}
        \centering
        \includegraphics[width=1.0\linewidth]{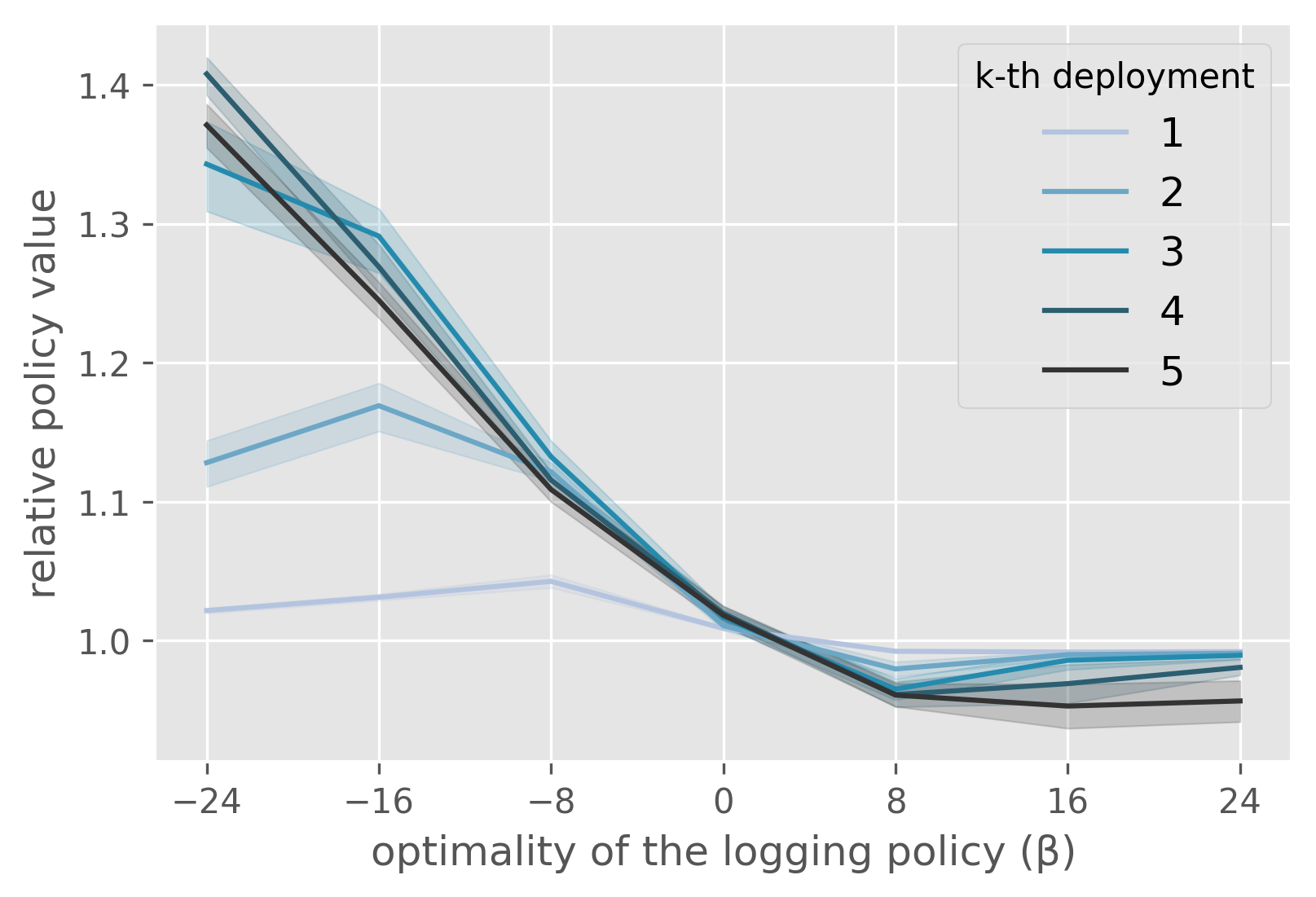}
        \textcolor{white}{xx} Relative Value
      \end{minipage}
      \begin{minipage}[b]{0.33\linewidth}
        \centering
        \includegraphics[width=1.0\linewidth]{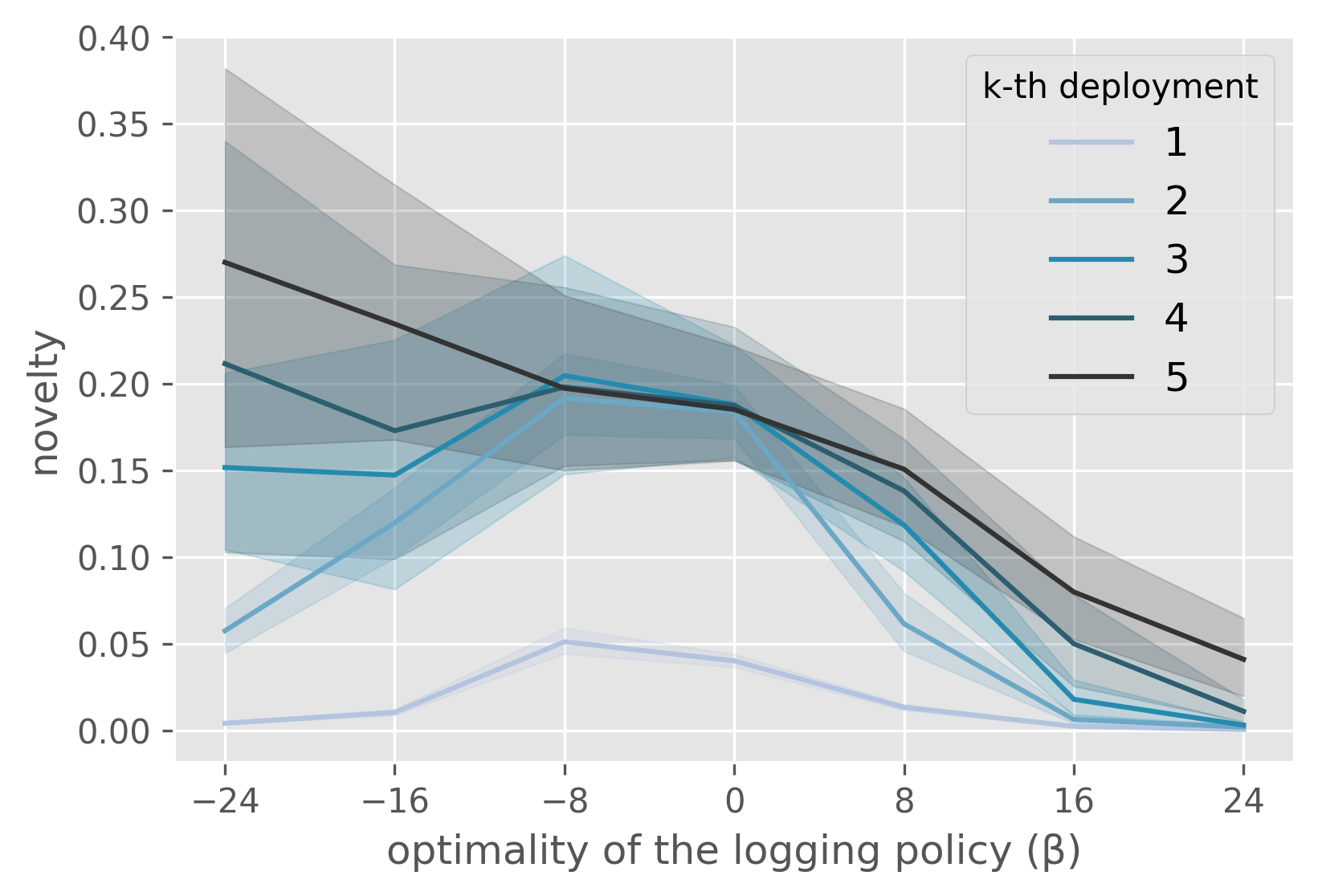}
        \textcolor{white}{xxx} Novelty
        \label{fig:one_stage_app}
      \end{minipage}
      \vspace{-2mm}
      \caption{Comparing relative value and novelty of the $k$-th deployment policies of DEPSUE ($K=5$).}
      \label{fig:5stage}
  \end{minipage}
\end{figure*}

\begin{figure*}[!htb]
\scalebox{0.95}{
\begin{tabular}{cccc}
\toprule
\textbf{logging policies} & \textbf{$\beta=-16$} & \textbf{$\beta=0$} & \textbf{$\beta=16$} \\ \midrule
\textbf{type of action} 
&
\multicolumn{3}{c}{
\begin{minipage}{0.75\hsize}
    \begin{center}
        \includegraphics[width=0.42\linewidth]{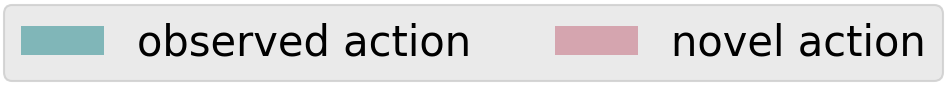}
    \end{center}
\end{minipage}
}
\\ \midrule \midrule
\textbf{ground-truth ($q$)}
&
\begin{minipage}{0.25\hsize}
    \begin{center}
        \includegraphics[clip, width=4.8cm]{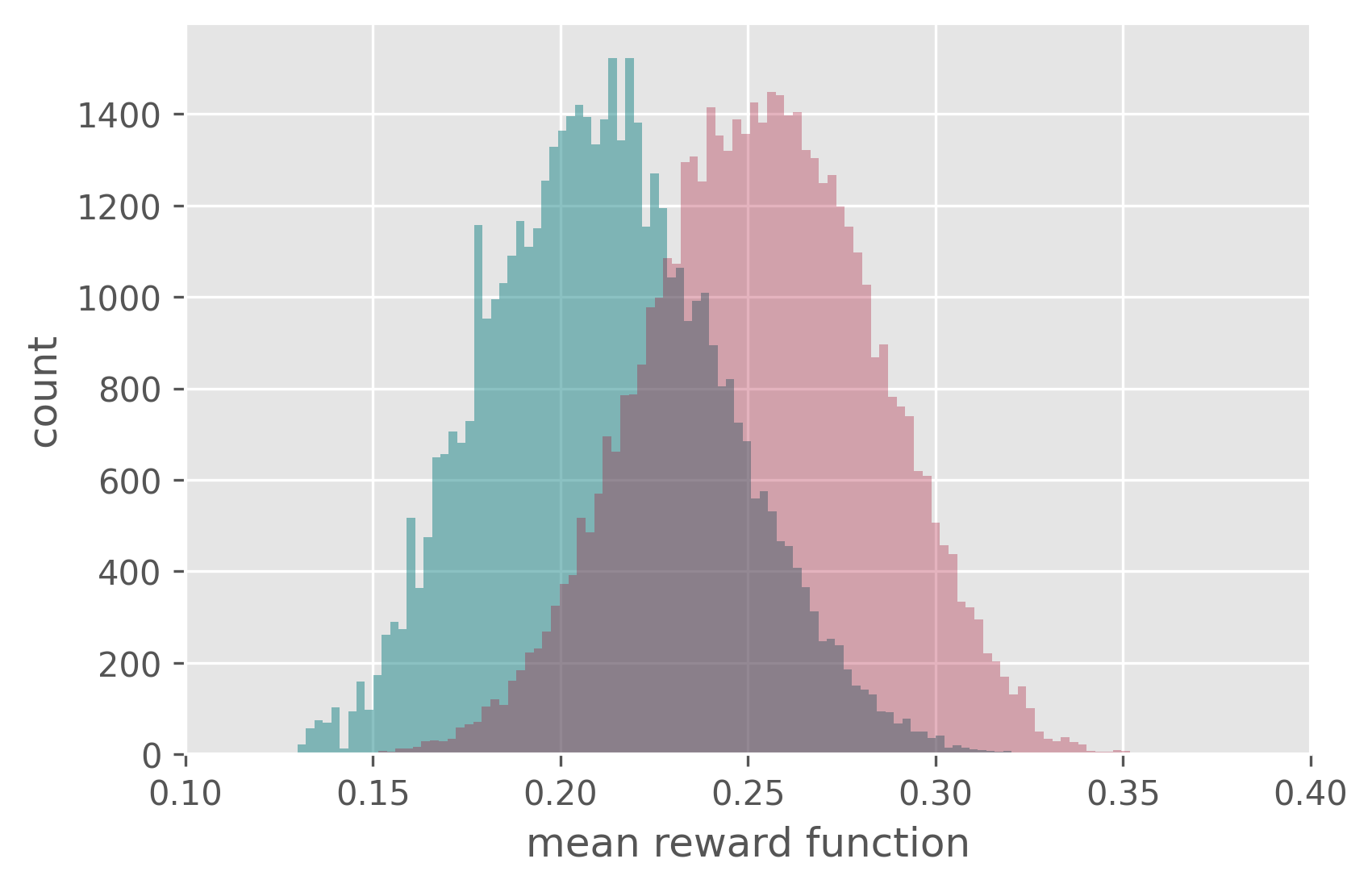}
    \end{center}
\end{minipage}
&
\begin{minipage}{0.25\hsize}
    \begin{center}
        \includegraphics[clip, width=4.8cm]{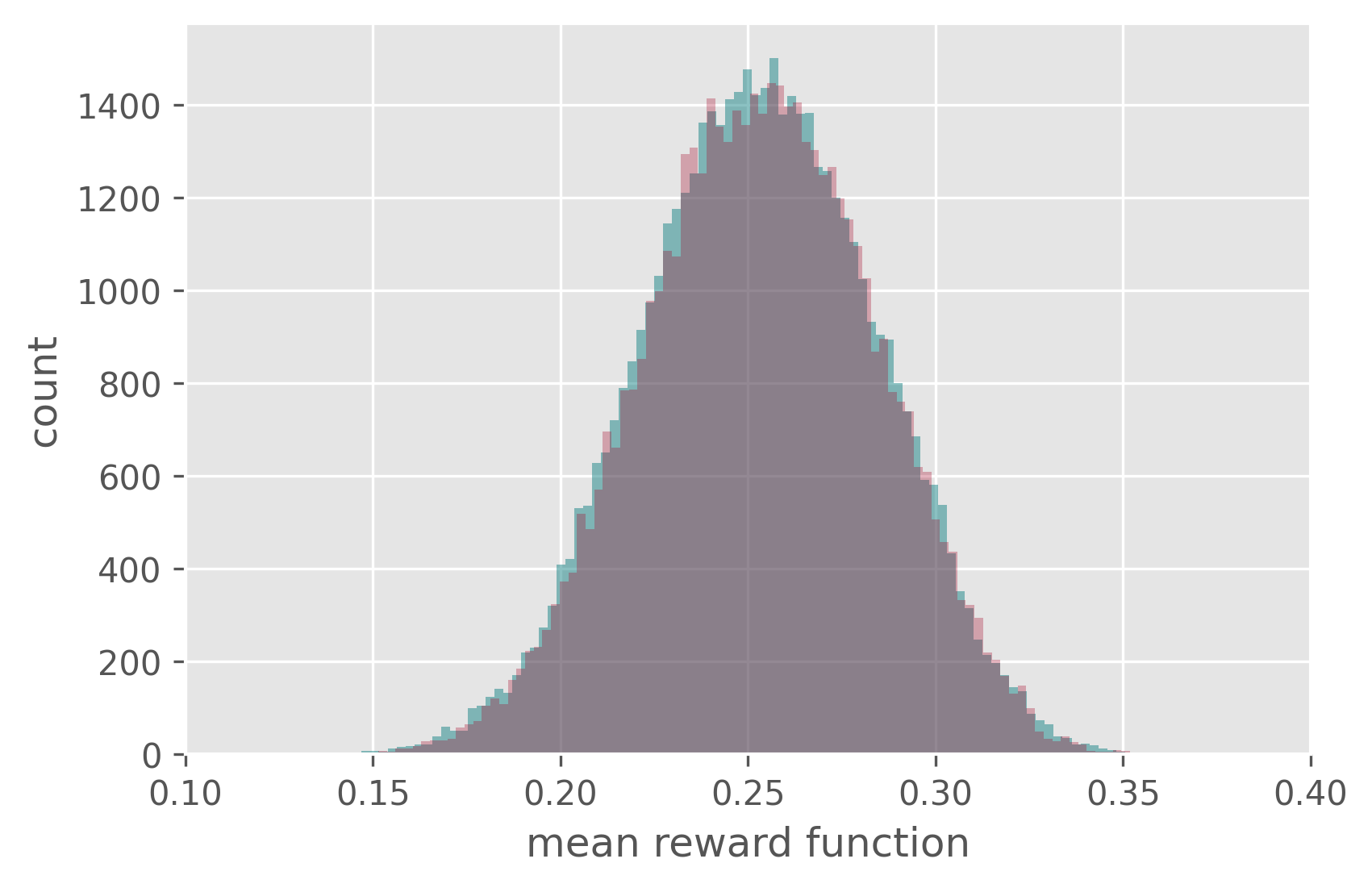}
    \end{center}
\end{minipage}
&
\begin{minipage}{0.25\hsize}
    \begin{center}
        \includegraphics[clip, width=4.8cm]{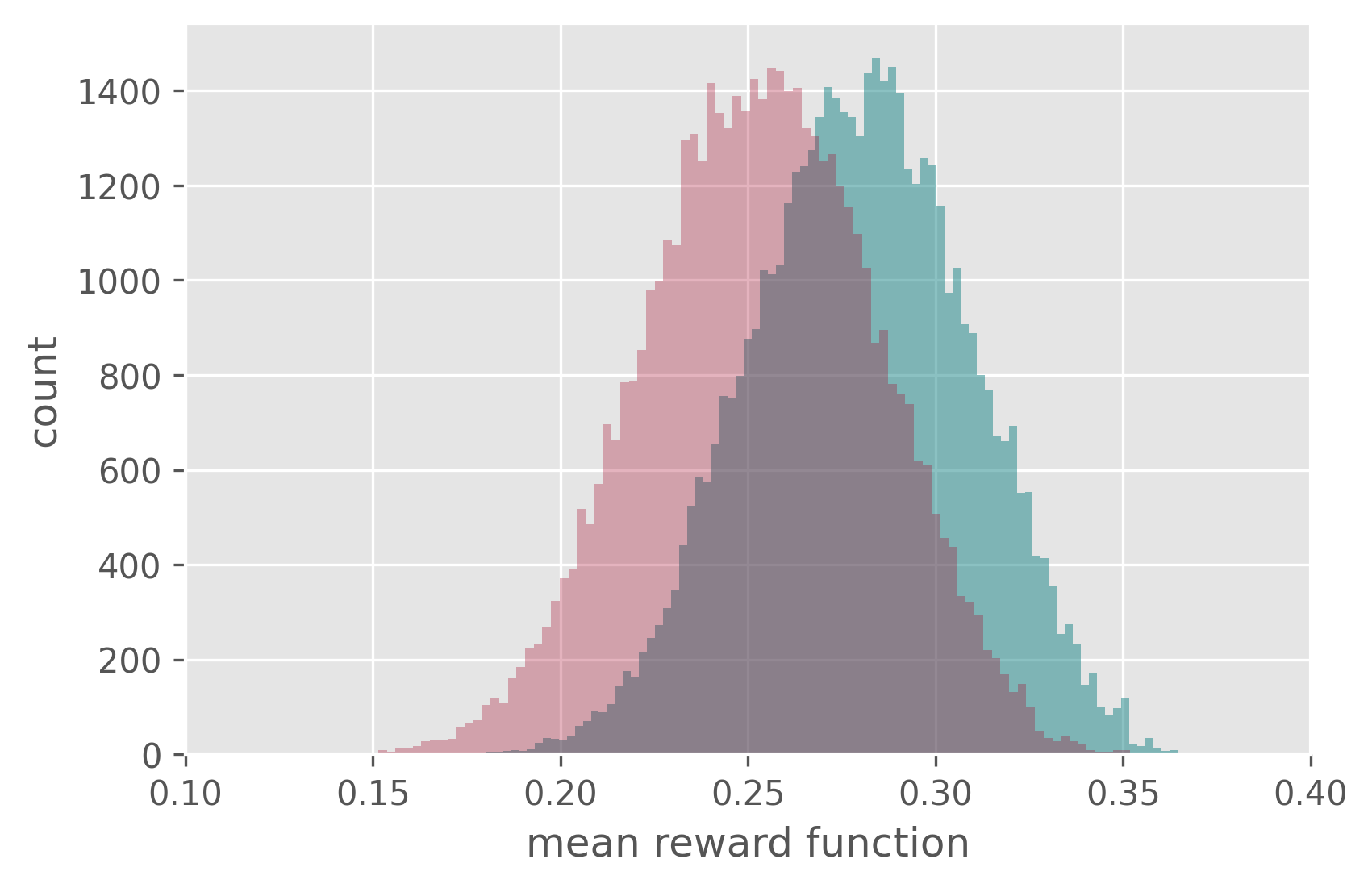}
    \end{center}
\end{minipage}
\\
& 
\multicolumn{3}{c}{
\begin{minipage}{0.75\hsize}
\begin{center}
\caption{Distribution of the ground-truth mean reward function ($q$)}
\label{fig:true_reward}
\end{center}
\end{minipage}
}
\\
\textbf{estimated w/ naive ($\hat{q}$)}
&
\begin{minipage}{0.25\hsize}
    \begin{center}
        \includegraphics[clip, width=4.8cm]{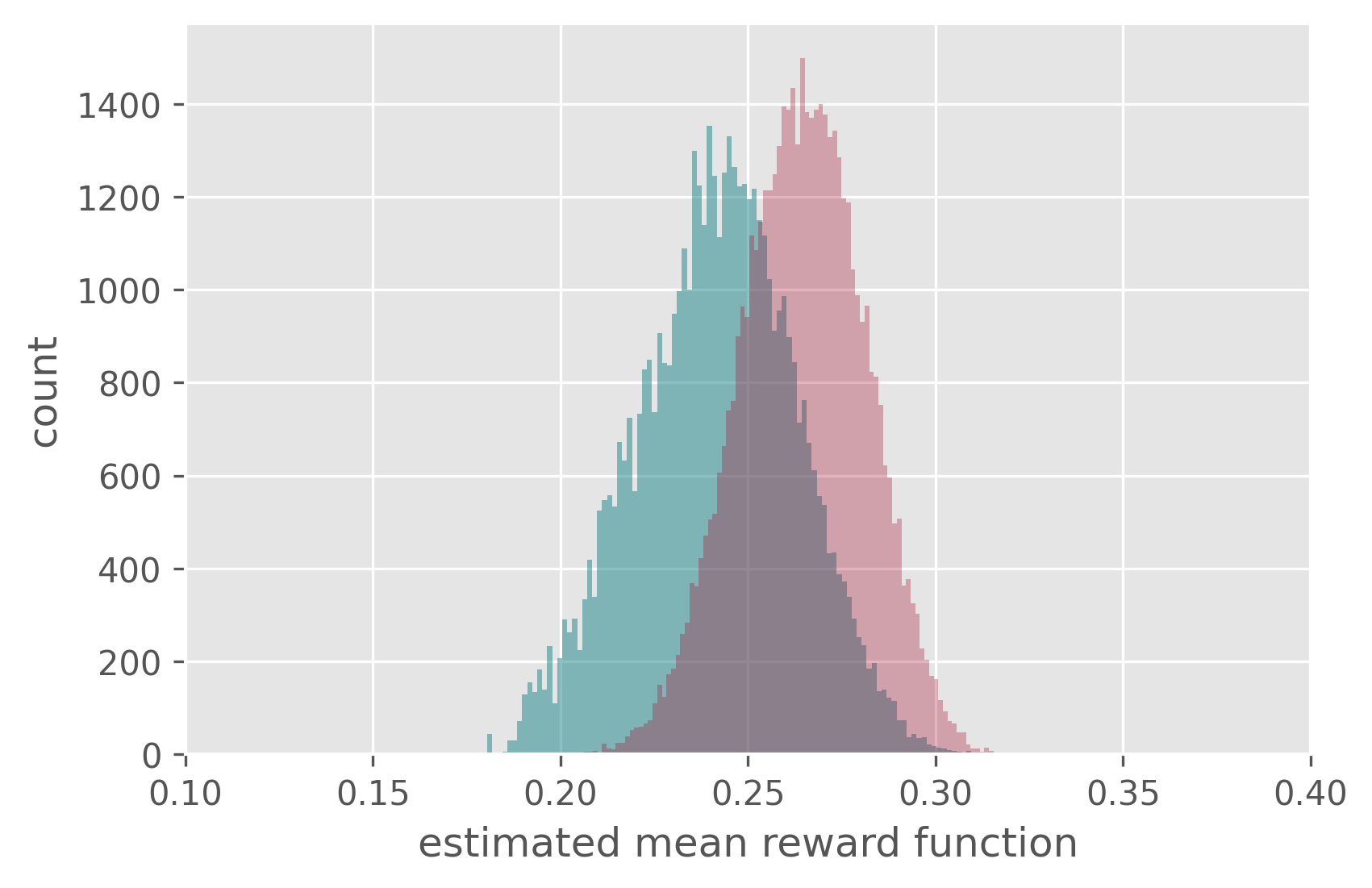}
    \end{center}
\end{minipage}
&
\begin{minipage}{0.25\hsize}
    \begin{center}
        \includegraphics[clip, width=4.8cm]{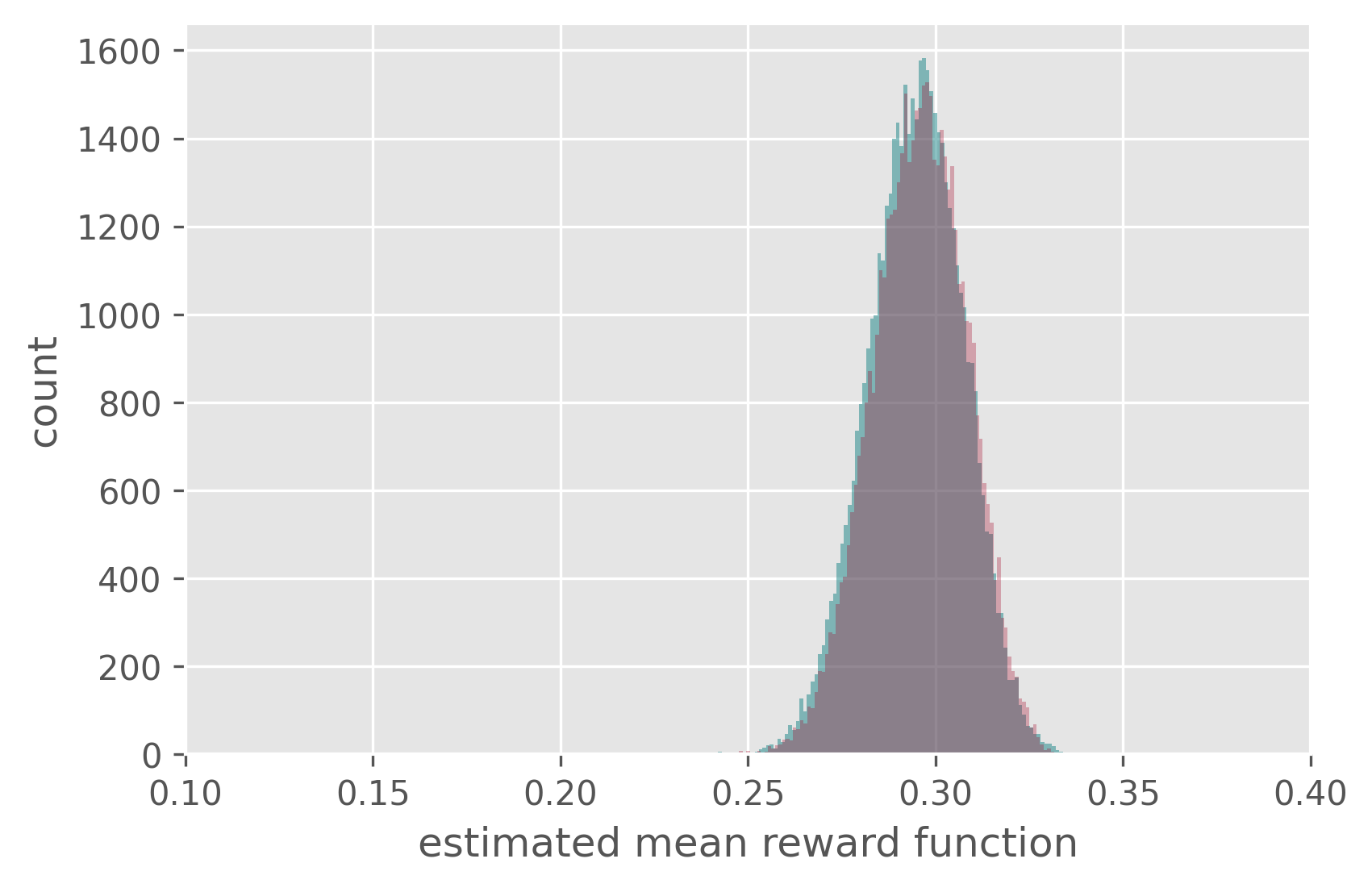}
    \end{center}
\end{minipage}
&
\begin{minipage}{0.25\hsize}
    \begin{center}
        \includegraphics[clip, width=4.8cm]{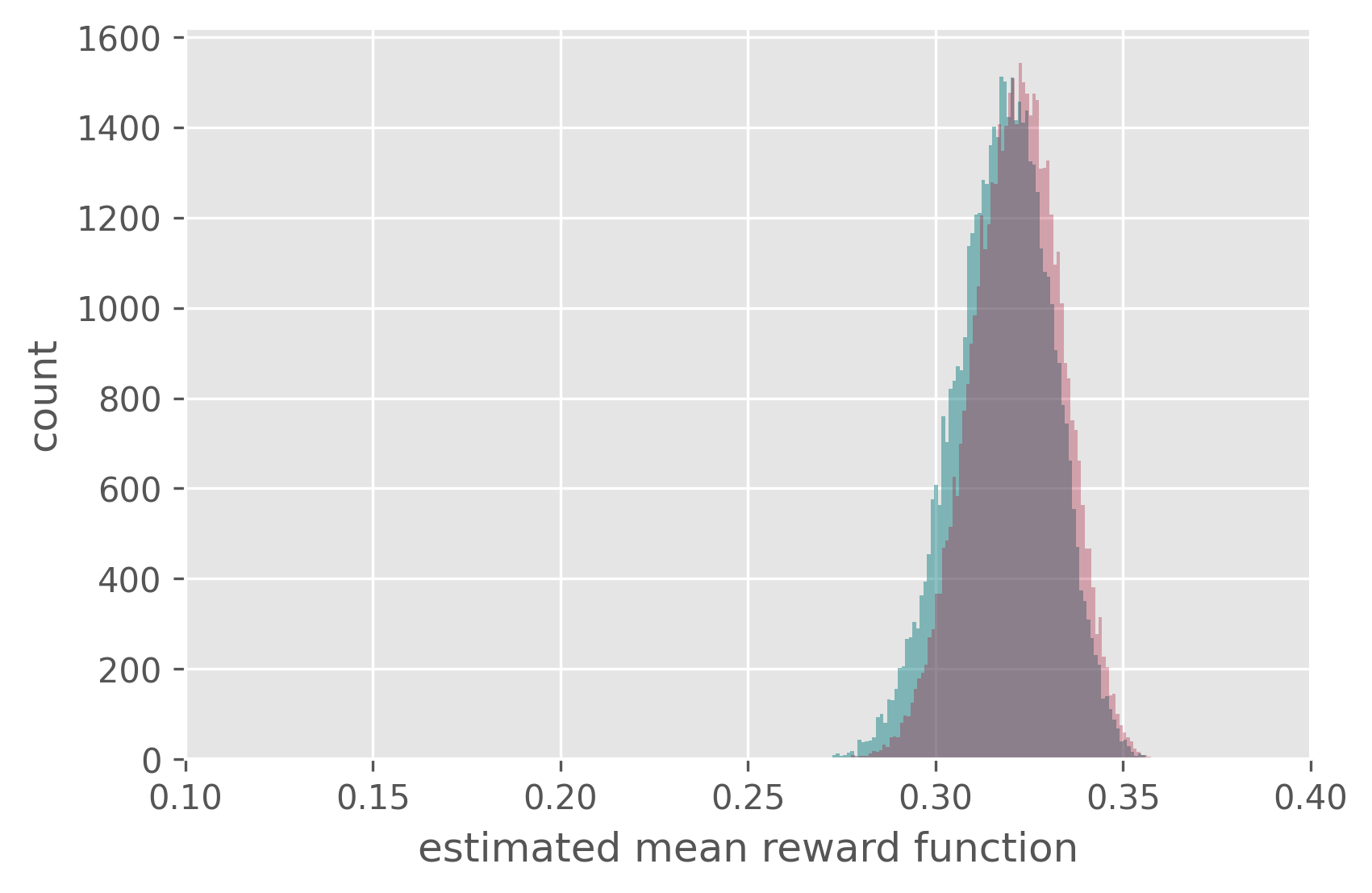}
    \end{center}
\end{minipage}
\\
&
\multicolumn{3}{c}{
\begin{minipage}{0.75\hsize}
\begin{center}
\caption{Distribution of the mean reward function ($\hat{q}$) estimated by the mean reward model}
\label{fig:mean_reward}
\end{center}
\end{minipage}
}
\\
\textbf{estimated w/ CQL ($\hat{q}$)}
&
\begin{minipage}{0.25\hsize}
    \begin{center}
        \includegraphics[clip, width=4.8cm]{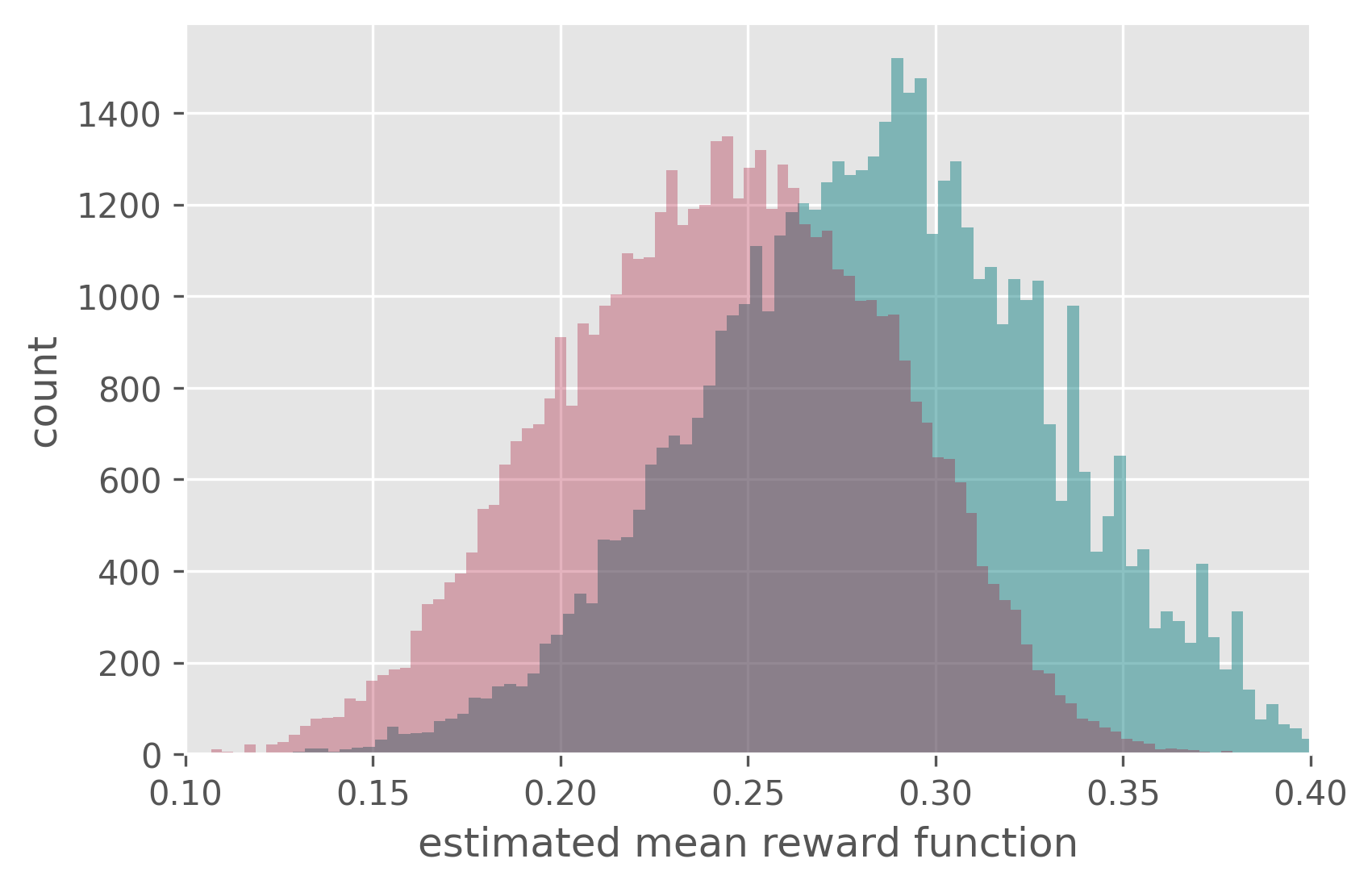}
    \end{center}
\end{minipage}
&
\begin{minipage}{0.25\hsize}
    \begin{center}
        \includegraphics[clip, width=4.8cm]{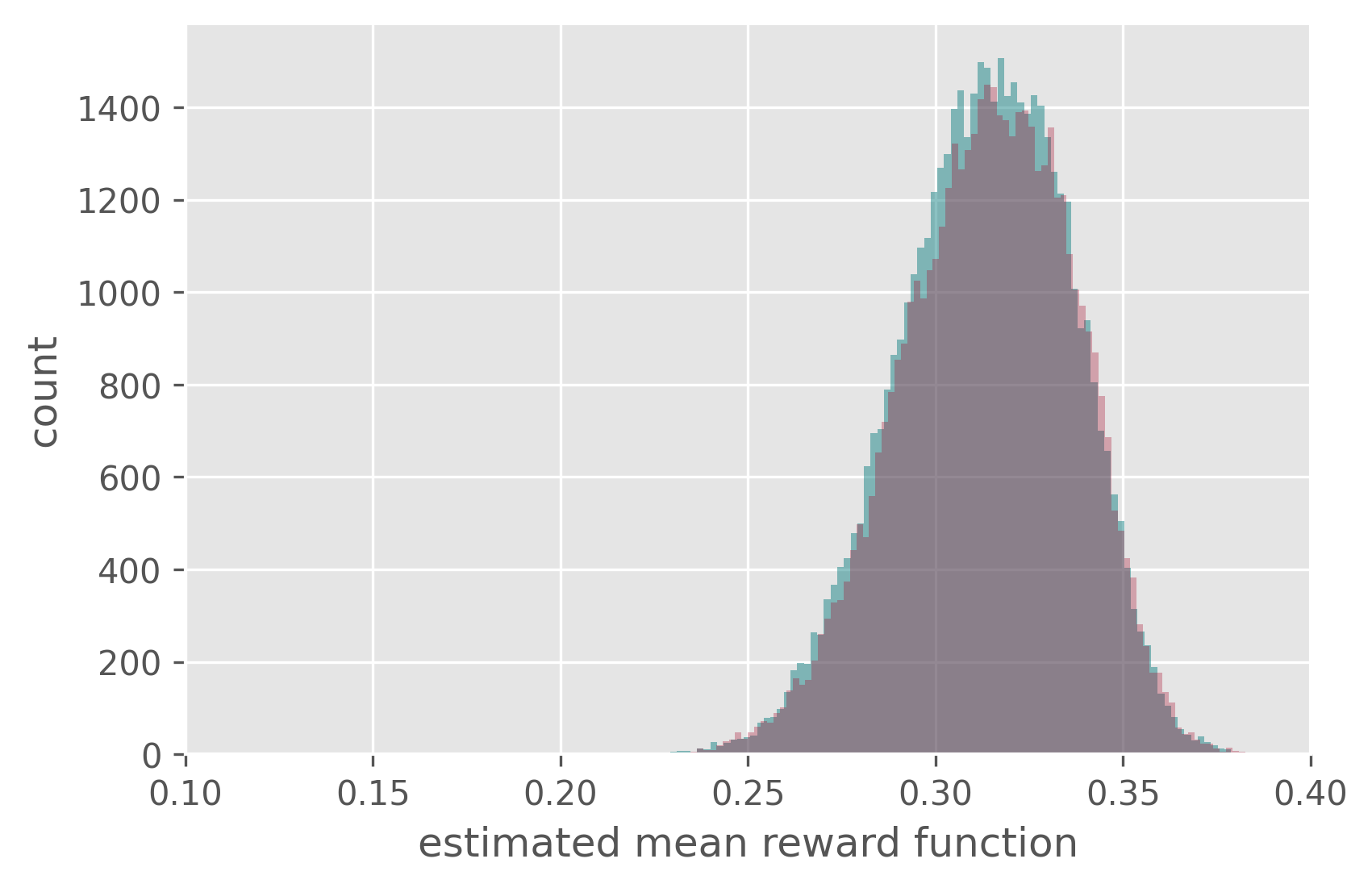}
    \end{center}
\end{minipage}
&
\begin{minipage}{0.25\hsize}
    \begin{center}
        \includegraphics[clip, width=4.8cm]{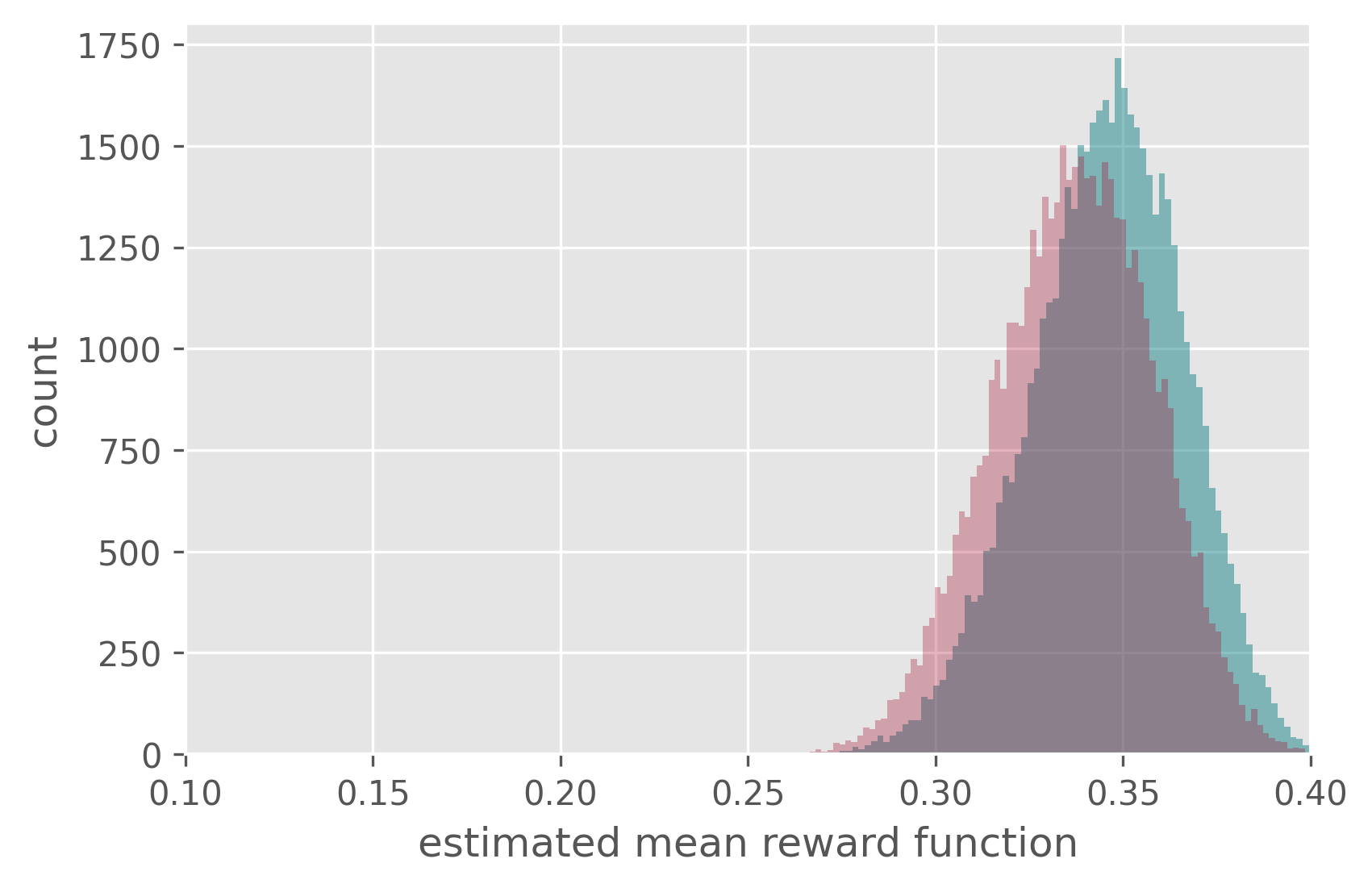}
    \end{center}
\end{minipage}
\\
& 
\multicolumn{3}{c}{
\begin{minipage}{0.75\hsize}
\begin{center}
\caption{Distribution of the mean reward function ($\hat{q}$) estimated by CQL}
\label{fig:cql_reward}
\end{center}
\end{minipage}
}
\\
\bottomrule
\end{tabular}
}
\end{figure*}


\end{document}